\documentclass[12pt]{article}

\usepackage{color}
\usepackage[utf8]{inputenc}
\usepackage[T1]{fontenc}    % use 8-bit T1 fonts
\usepackage{url}            % simple URL typesetting
\usepackage{booktabs}       % professional-quality tables
\usepackage{amsfonts}       % blackboard math symbols
\usepackage{nicefrac}       % compact symbols for 1/2, etc.
\usepackage{microtype}      % microtypography

\usepackage{amsmath}
\usepackage{amsfonts}
\usepackage{mathrsfs}
\usepackage{amssymb}
\usepackage{amsthm}
\usepackage{wrapfig}

\usepackage{thmtools}
\usepackage{thm-restate}

\usepackage{bm}

\usepackage{cleveref}

\usepackage{graphicx,subfigure}
\usepackage{tikz}
\usepackage{comment}

\usepackage{epstopdf}

\usepackage{libertine}

\usepackage{mathtools}
\DeclarePairedDelimiter\ceil{\lceil}{\rceil}

\usepackage{algorithm}
\usepackage{algorithmic}

\usepackage[numbers]{natbib}
\newdimen\arrowsize
\pgfarrowsdeclare{arcsq}{arcsq}
{
  \arrowsize=0.2pt
  \advance\arrowsize by .5\pgflinewidth
  \pgfarrowsleftextend{-4\arrowsize-.5\pgflinewidth}
  \pgfarrowsrightextend{.5\pgflinewidth}
}
{
  \arrowsize=1.5pt
  \advance\arrowsize by .5\pgflinewidth
  \pgfsetdash{}{0pt} % do not dash
  \pgfsetroundjoin   % fix join
  \pgfsetroundcap    % fix cap
  \pgfpathmoveto{\pgfpoint{0\arrowsize}{0\arrowsize}}
  \pgfpatharc{-90}{-140}{4\arrowsize}
  \pgfusepathqstroke
  \pgfpathmoveto{\pgfpointorigin}
  \pgfpatharc{90}{140}{4\arrowsize}
  \pgfusepathqstroke
}

\newtheorem{theorem}{Theorem}
\newtheorem{lemma}{Lemma}

\newtheorem{remark}{Remark}

\title{Why ResNet Works? Residuals Generalize}

\font\myfont=cmr12 at 13pt

\author{\myfont Fengxiang~He\thanks{UBTECH Sydney Artificial Intelligence Centre and School of Computer Science, Faculty of Engineering and Information Technologies, the University of Sydney, Darlington, NSW 2008, Australia. E-mail: fengxiang.he@sydney.edu.au, tongliang.liu@sydney.edu.au, and dacheng.tao@sydney.edu.au. First version in September 2018, second version in November 2018, and third version in February 2019.} \ \ \     Tongliang~Liu\footnotemark[1] \ \ \     Dacheng~Tao\footnotemark[1]}

\date{}

\begin{document}
\bibliographystyle{plain}

\maketitle

\begin{abstract}
  Residual connections significantly boost the performance of deep neural networks. However, there are few theoretical results that address the influence of residuals on the hypothesis complexity and the generalization ability of deep neural networks. This paper studies the influence of residual connections on the hypothesis complexity of the neural network in terms of the covering number of its hypothesis space. We prove that the upper bound of the covering number is the same as chain-like neural networks, if the total numbers of the weight matrices and nonlinearities are fixed, no matter whether they are in the residuals or not. This result demonstrates that residual connections may not increase the hypothesis complexity of the neural network compared with the chain-like counterpart. Based on the upper bound of the covering number, we then obtain an $\mathcal O(1 / \sqrt{N})$ margin-based multi-class generalization bound for ResNet, as an exemplary case of any deep neural network with residual connections. Generalization guarantees for similar state-of-the-art neural network architectures, such as DenseNet and ResNeXt, are straight-forward. From our generalization bound, a practical implementation is summarized: to approach a good generalization ability, we need to use regularization terms to control the magnitude of the norms of weight matrices not to increase too much, which justifies the standard technique of weight decay.
\end{abstract}

\newpage
\section{Introduction}
\label{introduction}

The recent years saw dramatic progress of deep neural networks \cite{lecun2015deep, greff2017lstm, shi2018face, silver2016mastering, litjens2017survey, chang2018generalization}. Since ResNet \cite{he2016deep}, residual connections have been widely used in many state-of-the-art neural network architectures \cite{he2016deep, huang2017densely, xie2017aggregated}, and lead a series of breakthroughs in computer vision \cite{krizhevsky2012imagenet, abadi2016tensorflow, lin2017feature, he2017mask, chen2018deeplab}, data mining \cite{witten2016data}, and so forth. Numerous empirical results are showing that residual connections can significantly ease the difficulty of training deep neural networks to fit the training sample while maintaining excellent generalization ability on test examples. However, little theoretical analysis has been presented on the effect of residual connections on the generalization ability of deep neural networks.

Residuals connect layers which are not neighboured in chain-like neural networks. These new constructions break the convention that stacking layers one by one to build a chain-like neural network. They introduce loops into neural networks, which are previously chain-like. Thus, intuitively, residual connections could significantly increase the complexity of the hypothesis space of the deep neural network, and therefore lead to a significantly worse generalization ability according to the principle of Occam's razor, which demonstrates a negative correlation between the generalization ability of an algorithm and its hypothesis complexity. Leaving this problem elusive could set restrictions on applying the recent progress of neural networks with residual connections to safety-critical domains, from autonomous vehicles \cite{janai2017computer} to medical diagnose \cite{esteva2017dermatologist}, in which algorithmic mistakes could lead to fatal disasters.

In this paper, we explore the influence on the hypothesis complexity induced by residual connections in terms of the covering number of the hypothesis space. An upper bound for the covering number is proposed. Our bound demonstrate that, when the total number of weight matrices involved in a neural network is fixed, the upper bound on the covering number remains the same, no matter whether the weight matrices are in the residual connections or in the ``stem''\footnote{The ``stem'' is defined to denote the chain-like part of the neural network besides all the residuals. For more details, please refer to Section \ref{stemVine}.}. This result indicates that residual connections may not increase the complexity of the hypothesis space compared with a chain-like neural network if the total numbers of the weight matrices and the non-linearities are fixed. Based on the upper bound on the covering number, we further prove an $\mathcal O(1 / \sqrt{N})$ generalization bound for ResNet as an exemplary case for all neural networks with residual connections, where $N$ is denoted to the training sample size. Based on our framework, generalization bounds for similar architectures constructed by adding residual connections to chain-like neural networks can be straightly obtained.

Our generalization bound closely depends on the product of the norms of all weight matrices. Specifically, there is a negative correlation between the generalization ability of a neural network with the product of the norms of all weight matrices. This feature leads to a practical implementation: 
\begin{quotation}
	{\it To approach a good generalization ability, we need to use regularization terms to control the magnitude of the norms of weight matrices.}
\end{quotation}
This implementation justifies the standard technique of weight decay in training deep neural networks, which uses the $L_{2}$ norm of the weights as a regularization term \cite{krogh1992simple}.

The rest of this paper is structured as follows. Section \ref{sec:review} reviews the existing literature regarding the generalization ability of deep neural networks in both theoretical and empirical aspects. Section \ref{preliminary} provides necessary preliminaries. Section \ref{stemVine} summarises the notation for deep neural networks with residual connections as the stem-vine framework. Section \ref{generalizationBound} presents our main results: a covering bound for deep neural networks with residual connections, a covering bound for ResNet, a generalization bound for ResNet, and a practical implementation from the theoretical results. Section \ref{sec:proof} collects all the proofs. And Section \ref{discussionConclusion} concludes this paper.

\section{Related Works}
\label{sec:review}

Understanding the generalization ability has vital importance to the development of deep neural networks. There already exist some results approaching this goal.

Zhang et al. conduct systematic experiments to explore the generalization ability of deep neural networks \cite{zhang2016understanding}. They show that neural networks can almost perfectly fit the training data even when the training labels are random. This paper attracts the community of learning theory to the important topic that how to theoretically interpret the success of deep neural networks.

Kawaguchi et al. discuss many open problems regarding the excellent generalization ability of deep neural networks despite the large capacity, complexity, possible algorithmic instability, nonrobustness, and sharp minima \cite{kawaguchi2017generalization}. They also provide some insights to solve the problems.

Harvey et al. prove upper and lower bounds on the VC-dimension of the hypothesis space of deep neural networks with the activation function of ReLU \cite{harvey2017nearly}. Specifically, the paper presents an $\mathcal O(W L \log(W))$ upper bound for the VC-dimension and an example of such networks with the VC-dimension $\Omega(W L\log(W/L))$, where $W$ and $L$ are respectively denoted to the width and depth of the neural network. The paper also gives a tight bound $\Theta (WU)$ for the VC-dimension of any deep neural network, where $U$ is the number of the hidden units in the neural network. The upper bounds of the VC-dimensions lead to an $\mathcal O(h/N)$ generalization bound, where $h$ is the VC-dimension and $N$ is the training sample size \cite{mohri2012foundations}.

Golowich et al. study the sample complexity of deep neural networks and present upper bounds on the Rademacher complexity of the neural networks in terms of the norm of the weight matrix in each layer \cite{golowich2017size}. Compared to previous works, these complexity bounds have improved dependence on the network depth, and under some additional assumptions, are fully independent of the network size (both depth and width). The upper bounds on the Rademacher complexity further lead to $\mathcal O(\frac{1}{\sqrt{N}})$ upper bounds on the generalization error of neural networks.

Neyshabur et al. explore several methods that could explain the generalization ability of deep neural networks, including norm-based control, sharpness, and robustness \cite{neyshabur2017exploring}. They study the potentials of these methods and highlight the importance of scale normalization. Additionally, they propose a definition of the sharpness and present a connection between the sharpness and the PAC-Bayes theory. They also demonstrate how well their theories can explain the observed experimental results.

Lang et al. explore the capacity measures for deep neural networks from a geometrical invariance viewpoint \cite{liang2017fisher}. They propose to use Fisher-Rao norm to measure the capacity of deep neural networks. Motivated by information geometry, they reveal the invariance property of the Fisher-Rao norm. The authors further establish some norm-comparison inequalities which demonstrate that the Fisher-Rao norm is an umbrella for many existing norm-based complexity measures. They also present experimental results to support their theoretical findings.

Novak et al. conduct comparative experiments to study the generalization ability of deep neural networks \cite{novak2018sensitivity}. The empirical results demonstrate that the input-output Jacobian norm and linear region counting play vital roles in the generalization ability of networks. Additionally, the generalization bound is also highly dependent on how close the output hypothesis is to the data manifold.

Two recent works respectively by Bartlett et al. \cite{bartlett2017spectrally} and Neyshabur et al. \cite{neyshabur2017pac} provide upper bounds for the generalization error of chain-like deep neural networks. Specifically, \cite{bartlett2017spectrally} proposes an $\mathcal O(1/\sqrt{N})$ spectral-normalized margin-based generalization bound by upper bounding the Rademacher complexity/covering number of the hypothesis space through the divide-and-conquer strategy. Meanwhile, \cite{neyshabur2017pac} obtains a similar result under the PAC-bayesian framework. Our work is partially motivated by the analysis in \cite{bartlett2017spectrally}.

Other advances include \cite{mhaskar2017and, shwartz2017opening, li2018visualizing, achille2018emergence, zhang2018information, soltanolkotabi2019theoretical}.

\section{Preliminary}
\label{preliminary}

In this section, we present the preliminaries necessary to develop our theory. It has two main parts: (1) important concepts to express the generalization capability of an algorithm; and (2) a margin-based generalization bound for multi-class classification algorithms. The preliminaries provide general tools for us to theoretically analyze multi-class classification algorithms.

Generalization bound is the upper bound of the generalization error which is defined as the difference between the expected risk (or, equivalently, the expectation of test error) of the output hypothesis of an algorithm and the corresponding empirical risk (or, equivalently, the training error).\footnote{Some works define generalization error as the expected error of an algorithm (see, e.g., \cite{mohri2012foundations}). As the training error is fixed when both training data and the algorithm are fixed, this difference in definitions can only lead to a tiny difference in results. In this paper, we select one for the brevity and would not limit any generality.} Thus, the generalization bound quantitatively expresses the generalization capability of an algorithm.

As indicated by the principle of Occam's razor, there is a negative correlation between the generalization capability of an algorithm and the complexity of the hypothesis space that the algorithm can compute. Two fundamental measures for the complexity are {\it VC dimension} and {\it Rademacher complexity} (see, respectively, \cite{vapnik1974theory} and \cite{bartlett2002rademacher}). Furthermore, they can be upper bounded by another important complexity {\it covering number} (see, respectively, \cite{dudley2010sizes} and \cite{haussler1995sphere}). Recent advances include local Rademacher complexity and algorithmic stability (see, respectively, \cite{bartlett2005local} and \cite{bousquet2001algorithmic,liu2017algorithmic}). These theoretical tools have been widely applied to analyze many algorithms (see, e.g., \cite{li2018better, han2018matrix, meng2017generalization, tian2018eigenfunction}).

To formally formularise the problem, we first define the {\it margin operator} $\mathcal M$ for the $k$-class classification task as
\begin{align}
	\mathcal M: ~ & \mathbb R^{k} \times \{ 1, \ldots, k \} \to \mathbb R ~, \nonumber\\
	& (v, y) \mapsto v_{y} - \max_{i \ne y} v_{i} ~.
\end{align}
Then, {\it ramp loss} $l_{\lambda}: \mathbb R \to \mathbb R^{+}$ is defined as
\begin{equation}
	l_{\lambda}(r) = \begin{cases}
		0, & r < -\lambda ~,\\
		1 + r/\lambda, & -\lambda \le r \le 0 ~,\\
		1, & r > 0 ~.
	\end{cases}
\end{equation}
Furthermore, given a hypothesis function $F: \mathbb R^{n_{0}} \to \mathbb R^{k}$ for the $k$-class classification, {\it empirical ramp risk} on a dataset $D = \{(x_{1}, y_{1}), \ldots, (x_{n}, y_{n})\}$ is defined as
\begin{equation}
	\hat{\mathcal R}_{\lambda}(F) = \frac{1}{n} \sum_{i = 1}^{n}(l_{\lambda}(-\mathcal M(F(x_{i}), y_{i}))) ~.
\end{equation}
Empirical ramp risk $\hat{\mathcal R}_{\lambda}(F)$ expresses the training error of the hypothesis function $F$ on the dataset $D$.

Meanwhile, the expected risk (and also, equivalently, the expected test error) of the hypothesis function $F$ under $0$-$1$ loss is
\begin{equation}
	\Pr\{ \arg\max_{i} F(x)_{i} \ne y \} ~,
\end{equation}
where $x$ is an arbitrary feature, $y$ is the corresponding correct label, and the probability is in term of the pair $(x, y)$.

Suppose a {\it hypothesis space} $\mathcal H|_{D}$ is constituted by all hypothesis functions that can be computed by a neural network trained on a dataset $D$. The {\it empirical Rademacher complexity} of the hypothesis space $\mathcal H|_{D}$ is defined as
\begin{equation}
	\hat{\mathfrak R}(\mathcal H|_{D}) = \mathbb E_{\bm{\epsilon}} \left[ \sup_{F \in \mathcal H|_{D}} \frac{1}{n}  \sum_{i = 1}^{n} \epsilon_{i} F(x_{i}, y_{i}) \right] ~,
\end{equation}
where $\bm{\epsilon} = (\epsilon_{1}, \ldots, \epsilon_{n})$ and $\epsilon_{i}$ is a uniform variable on $\{ -1, +1 \}$. A margin-based bound for multi-class classifiers is given as the following lemma.
\begin{lemma} [see \cite{bartlett2017spectrally}, Lemma 3.1]
\label{RadBound}
	Given a function set $\mathcal H$ that $\mathcal H \ni F: ~ \mathbb R^{n_{0}} \to \mathbb R^{n_{L}}$ and any margin $\lambda > 0$, define
	\begin{equation}
		\mathcal H_{\lambda} \triangleq \{ (x, y) \mapsto l_{\lambda}(-\mathcal M(F(x), y)): F \in \mathcal H \} ~.
	\end{equation}
	Then, for any $\delta \in (0, 1)$, with probability at least $1 - \delta$ over a dataset $D$ of size $n$, every $F \in \mathcal H|_{D}$ satisfies
	\begin{align}
	\label{formulaRadBound}
		\Pr\{ \arg\max_{i} F(x)_{i} \ne y \} - \hat{\mathcal R}_{\lambda}(F) \le 2 \hat{\mathfrak R}(\mathcal H_{\lambda}|_{D}) + 3 \sqrt{\frac{\log(1/\delta)}{2n}} ~.
	\end{align}
\end{lemma}

This generalization bound is developed by employing Rademacher complexity which is upper bounded by covering number (see, respectively, \cite{zhang2002covering,zhang2004statistical} and \cite{haussler1995sphere,mohri2012foundations}). A detailed proof can be found in \cite{bartlett2017spectrally}. Lemma \ref{RadBound} relates the generalization capability (expressed by $\Pr\{ \arg\max_{i} F(x)_{i} \ne y \} - \hat{\mathcal R}_{\lambda}(F)$) to the hypothesis complexity (expressed by $\hat{\mathfrak R}(\mathcal H_{\lambda}|_{D})$). It suggests that if one can find an upper bound for empirical Rademacher complexity, an upper bound of generalization error can be straightly obtained. Bartlett et al. give a lemma that bounds empirical Rademacher complexity via upper bounding covering number \cite{bartlett2017spectrally} derived from the Dudley entropy integral bound \cite{dudley2010sizes,dudley2010universal}. Specifically, if the $\varepsilon$-covering number $\mathcal N(\mathcal H_{\lambda}|_{D}, \varepsilon, \|\cdot\|)$ is defined as the minimum number of the balls with radius $\varepsilon > 0$ needed to cover the space $\mathcal H_{\lambda}|_{D}$ with a norm $\|\cdot\|$, the lemma is as follows.

\begin{lemma}[see \cite{bartlett2017spectrally}, Lemma A.5]
\label{covNumBound}
	Suppose $\bm 0 \in \mathcal H_{\lambda}$ while all conditions in Lemma \ref{RadBound} hold. Then
	\begin{align}
	\label{formulaCovNumBound}
		\hat{\mathfrak R}(\mathcal H_{\lambda}|_{D}) \le \inf_{\alpha > 0} \left( \frac{4\alpha}{\sqrt{n}} + \frac{12}{n} \int_{\alpha}^{\sqrt n} \sqrt{\log \mathcal N(\mathcal H_{\lambda}|_{D}, \varepsilon, \| \cdot |_{2})} \text{d}\varepsilon \right) ~.
	\end{align}
\end{lemma}

Combining Lemmas \ref{RadBound} and \ref{covNumBound}, we relate the covering bound of an algorithm to the generalization bound of the algorithm. In the rest of this paper, we develop generalization bounds for deep neural networks with residual connections via upper bounding covering numbers.

To avoid technicalities, the measurability/integrability issues are ignored throughout this paper. Moreover, Fubini's theorem is assumed to be applicable for any integration with respect to multiple variables, that the order of integrations is exchangeable.

\section{Stem-Vine Framework}
\label{stemVine}

This section provides a notation system for deep neural networks with residual connections. Motivated by the topological structure, we call it the {\it stem-vine framework}.

In general, deep neural networks are constructed by connecting many weight matrices and nonlinear operators (nonlinearities), including ReLU, sigmoid, and max-pooling. In this paper, we consider a neural network constructed by adding multiple residual connections to a ``chain-like'' neural network that stacks a series of weight matrices and nonlinearities forward one by one. Motivated by the topological structure, we call the chain-like part as the {\it stem} of the neural network and call the residual connections as the {\it vines}. Both stems and vines themselves are constructed by stacking multiple weight matrices and nonlinearities. 

We denote the weight matrices and the nonlinearities in the stem $S$ respectively as
\begin{gather}
	A_{i} \in \mathbb R^{n_{i-1} \times n_{i}} ~,\\
	\sigma_{j}: \mathbb R^{n_{j}} \to \mathbb R^{n_{j}} ~,
\end{gather}
where $i = 1, \ldots, L$, $L$ is the number of weight matrices in the stem, $j = 1, \ldots, L_{N}$, $L_{N}$ is the number of nonlinearities in the stem, $n_{i}$ is the dimension of the output of the $i$-th weight matrix, $n_{0}$ is the dimension of the input data to the network, and $n_{L}$ is the dimension of the output of the network. Thus we can write the stem $S$ as a vector to express the chain-like structure. Here for the simplicity and without any loss of the generality, we give an example that the numbers of weight matrices and nonlinearities are equal\footnote{If two weight matrices, $A_{i}$ and $A_{i+1}$, are connected directly without a nonlinearity between them, we define a new weight matrix $A = A_{i} \cdot A_{i+1}$. The situations that nonlinearities are directly connected are similar, as the composition of any two nonlinearities is still a nonlinearity.

Meanwhile, the number of the weight matrices does not necessarily equal the number of nonlinearities. Sometimes, if a vine connects the stem at a vertex between two weight matrices (or two nonlinearities), the number of the weight matrices (nonlinearities) would be larger than the number of nonlinearities (weight matrices). Taken the $34$-layer ResNet as an example, a vine connects the stem between two nonlinearities $\sigma_{33}$ and $\sigma_{34}$. In this situation, we cannot merge the two nonlinearities, so the number of the nonlinearities is larger than the number of weight matrices.
}
, i.e., $L_{N} = L$, as the following equation,
\begin{equation}
\label{stemStructure}
	S = (A_{1}, \sigma_{1}, A_{2}, \sigma_{2}, \ldots, A_{L}, \sigma_{L}) ~.
\end{equation}

For the brevity, we give an index $j$ to each vertex between a weight matrix and a nonlinearity and denote the $j$-th vertex as $N(j)$. Specifically, we give the index $1$ to the vertex that receives the input data and $L + L_{N} + 1$ to the vertex after the last weight matrix/nonlinearity. Taken eq. (\ref{stemStructure}) as an example, the vertex between the nonlinearity $\sigma_{i-1}$ and the weight matrix $A_{i}$ is denoted as $N(2i - 1)$ and the vertex between the weight matrix $A_{i}$ and the nonlinearity $\sigma_{i}$ is denoted as $N(2i)$.

Vines are constructed to connect the stem at two different vertexes. And there could be over one vine connecting a same pair of the vertexes. Therefore, we use a triple vector $(s, t, i)$ to index the $i$-th vine connecting the vertexes $N(s)$ and $N(t)$ and denote the vine as $V(s, t, i)$. All triple vectors $(s, t, i)$ constitute an index set $I_{V}$, i.e., $(s, t, i) \in I_{V}$. Similar to the stem, each vine $V(s, t, i)$ is also constructed by a series of weight matrices $A^{s, t, i}_{1}, \ldots, A^{s, t, i}_{L^{s, t, i}}$ and nonlinearities $\sigma^{s, t, i}_{1}, \ldots, \sigma^{s, t, i}_{L^{s, t, i}_{N}}$, where $L^{s, t, i}$ is the number of weight matrices in the vine, while $L^{u,v,i}_{N}$ is the number of the nonlinearities.

\begin{figure}
%\begin{center}
\centering
\includegraphics[width=5in]{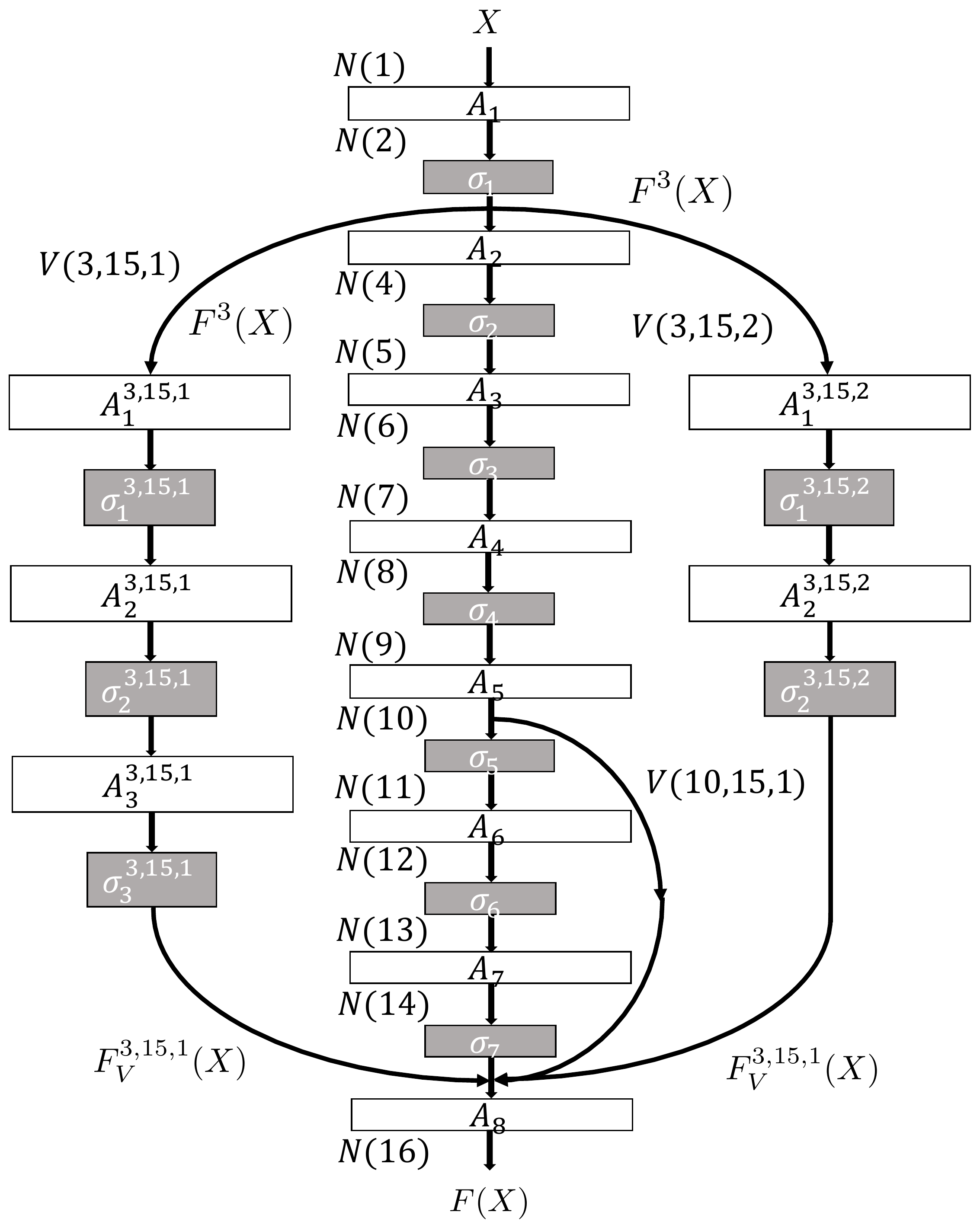}%
\caption{A Deep Neural Network with Residual Connections under the Stem-Vine Framework.}
\label{stemVineFrame}
%\end{center}
\end{figure}

Multiplying by a weight matrix corresponds to an affine transformation on the data matrix. Also, nonlinearities induce nonlinear transformations. Through a series of affine transformations and nonlinear transformations, hierarchical features are extracted from the input data by neural networks. Usually, we use the {\it spectrum norms} of weight matrices and the {\it Lipschitz constants} of a nonlinearities to express the intensities respectively of the affine transformations and the nonlinear transformations. We call a function $f(x)$ is $\rho$-Lipschitz continuous if for any $x_{1}$ and $x_{2}$ in the support domain of $f(x)$, it holds that
\begin{equation}
	\| f(x_{1}) - f(x_{2}) \|_{f} \le \rho \| x_{1} - x_{2} \|_{x} ~,
\end{equation}
where $\| \cdot \|_{f}$ and $\| \cdot \|_{x}$ are respectively the norms defined on the spaces of $f(x)$ and $x$. Fortunately, almost all nonlinearities normally used in neural networks are Lipschitz continuous, such as ReLU, max-pooling, and sigmoid (see \cite{bartlett2017spectrally}).

Many important tasks for deep neural networks can be categorized into multi-class classification. Suppose input examples $z_{1} \ldots, z_{n}$ are given, where $z_{i} = (x_{i}, y_{i})$, $x_{i} \in \mathbb R^{n_{0}}$ is an instance, $y \in \{1, \ldots, n_{L}\}$ is the corresponding label, and $n^{L}$ is the number of the classes. Collect all instances $x_{1}, \ldots, x_{n}$ as a matrix $X = \left(x_{1}, \ldots, x_{n}\right)^{T} \in \mathbb R^{n \times n_{0}}$ that each row of $X$ represents a data point. By employing optimization methods (usually stochastic gradient decent, SGD), neural networks are trained to fit the training data and then predict on test data. In mathematics, a trained deep neural network with all parameters fixed computes a hypothesis function $F: \mathbb R^{n_{0}} \to \mathbb R^{n_{L}}$. And a natural way to convert $F$ to a multi-class classifier is to select the coordinate of $F(x)$ with the largest magnitude. In other words, for an instance $x$, the classifier is $x \mapsto \arg\max_{i} F(x)_{i}$. Correspondingly, the {\it margin} for an instance $x$ labelled as $y_{i}$ is defined as $F(x)_{y} - \max_{i \ne y} F(x)_{i}$. It quantitatively expresses the confidence of assigning a label to an instance.

To express $F$, we first define the functions respectively computed by the stem and vines. Specifically, we denote the function computed by a vine $V(s, t, i)$ as:
\begin{align}
%\label{funcVine}
	F_{V}^{s, t, i}(X) = \sigma^{u,v,i}_{L^{u,v,i}}(A^{u,v,i}_{L^{u,v,i}} \sigma^{u,v,i}_{L^{u,v,i}-1} ( \ldots \sigma_{1}(A^{u,v,i}_{1} X) \ldots )) ~.
\end{align}
Similarly, the stem computes a function as the following equation:
\begin{equation}
%\label{funcWholeBaraStem}
	F_{S}(X) = \sigma_{L}(A_{L} \sigma_{L-1} (\ldots \sigma_{1}(A_{1} X) \ldots)) ~.
\end{equation}
Furthermore, we denote the output of the stem at the vertex $N(j)$ as the following equation:
\begin{equation}
%\label{funcPartBaraStem}
	F_{S}^{j}(X) = \sigma_{j}(A_{j} \sigma_{j-1} (\ldots \sigma_{1}(A_{1} X) \ldots)) ~.
\end{equation}
$F_{S}^{j}(X)$ is also the input of the rest part of the stem. Eventually, with all residual connections, the output hypothesis function $F^{j}(X)$ at the vertex $N(j)$ is expressed by the following equation:
\begin{equation}
\label{hypothesisFunctionJ}
	F^{j}(X) = F_{S}^{j}(X) + \sum_{(u,j,i) \in I_{V}} F_{V}^{u,j,i} (X) ~.
\end{equation}
Apparently,
\begin{equation}
	F_{S}(X) = F_{S}^{L}(X),~F(X) = F^{L}(X) ~.
\end{equation}

Naturally, we call this notation system as the {\it stem-vine framework}, and Figure \ref{stemVineFrame} gives an example.

\section{Generalization Bound}
\label{generalizationBound}

In this section, we study the generalization capability of deep neural networks with residual connections and provide a generalization bound for ResNet as an exemplary case. This generalization bound is derived upon the margin-based multi-class bound given by Lemmas \ref{RadBound} and \ref{covNumBound} in Section \ref{preliminary}. Indicated by Lemmas \ref{RadBound} and \ref{covNumBound}, a natural way to approach the generalization bound is to explore the covering number of the corresponding hypothesis space. Motivated by this intuition, we first propose an upper bound of the covering number (or briefly, {\it covering bound}) generally for any deep neural networks under the stem-vine framework. Then, as an exemplary case, we obtain a covering bound for ResNet. Applying Lemmas \ref{RadBound} and \ref{covNumBound}, a generalization bound for ResNet is eventually presented. The proofs for covering bounds will be given in Section \ref{sec:proof}.

As a convention, when we introduce a new structure to boost the training performance (including training accuracy, training time, etc.), we should be very careful to prevent the algorithm from overfitting (which manifests itself as an unacceptably large generalization error). ResNet introduces ``loops'' into chain-like neural networks by residual connections, and therefore becomes a more complex model. Empirical results indicate that the residual connections significantly reduce the training error and accelerate the training speed, while maintains generalization capability at the same time. However, there is so far no theoretical evidence to explain/support the empirical results.

Our result in covering bound indicates that when the total number of weight matrices is fixed, no matter where the weight matrices are (either in the stem or in the vines, and even when there is no vine at all), the complexities of the hypothesis spaces that computed by deep neural networks remain invariant. Combing various classic results in statistical learning theories (Lemmas \ref{RadBound} and \ref{covNumBound}), our results further indicate that the generalization capability of deep neural networks with residual connections could be as equivalently good as the ones without any residual connection at least in the worst cases. Our theoretical result gives an insight into why the deep neural networks with residual connections have equivalently good generalization capability compared with the chain-like ones while having competitive training performance.

\subsection{Covering Bound for Deep Neural Networks with Residuals}
\label{subSecCoverBoundGeneral}

In this subsection, we give a covering bound generally for any deep neural network with residual connections.

\begin{theorem}[Covering Bound for Deep Neural Network]
\label{coverBoundGeneral}
	Suppose a deep neural network is constituted by a stem and a series of vines.
	
	For the stem, let $(\varepsilon_{1}, \ldots, \varepsilon_{L})$ be given, along with $L_{N}$ fixed nonlinearities $(\sigma_{1}, \ldots, \sigma_{L_{N}})$. Suppose the $L$ weight matrices $(A_{1}, \ldots, A_{L})$ lies in $\mathcal B_{1} \times \ldots \times \mathcal B_{L}$, where $\mathcal B_{i}$ is a ball centered at $0$ with radius of $s_{i}$, i.e., $\| A_{i} \| \le s_{i}$. Suppose the vertex that directly follows the weight matrix $A_{i}$ is $N(M(i))$ ($M(i)$ is the index of the vertex). All $M(i)$ constitute an index set $I_{M}$. When the output $F_{M(j-1)}(X)$ of the weight matrix $A_{j-1}$ is fixed, suppose all output hypotheses $F_{M(j)}(X)$ of the weight matrix $A_{j}$ constitute a hypothesis space $\mathcal H_{M(j)}$ with an $\varepsilon_{M(j)}$-cover $\mathcal W_{M(j)}$ with covering number $\mathcal N_{M(j)}$. Specifically, we define $M(0)=0$ and $F_{0}(X) = X$.

Each vine $V(u,v,i)$, $(u,v,i) \in I_{V}$ is also a chain-like neural network that constructed by multiple weight matrices $A^{u,v,i}_{j}$, $j \in \{1, \ldots, L^{u,v,i}\}$, and nonlinearities $\sigma^{u,v,i}_{j}$, $j \in \{1, \ldots, L_{N}^{u,v,i}\}$. Suppose for any weight matrix $A^{u,v,i}_{j}$, there is a $s_{j}^{u,v,i} > 0$ such that $\| A^{u,v,i}_{j} \|_{\sigma} \le s_{j}^{u,v,i}$. Also, all nonlinearities $\sigma^{u,v,i}_{j}$ are Lipschitz continuous. Similar to the stem, when the input of the vine $F_{u}(X)$ is fixed, suppose the vine $V(u, v, i)$ computes a hypothesis space $\mathcal H^{u,v,i}_{V}$, constituted by all hypotheses $F_{V}^{u,v,i}(X)$, has an $\varepsilon_{u,v,i}$-cover $\mathcal W_{V}^{u,v,i}$ with covering number $\mathcal N^{u,v,i}_{V}$.

Eventually, we denote the hypothesis space computed by the neural network is $\mathcal H$. Then there exists an $\varepsilon$ in terms of $\varepsilon_{i}$, $i = \{ 1, \ldots, L \}$ and $\varepsilon_{u,v,i}$, $(u,v,i) \in I_{V}$, such that the following inequality holds:
\begin{align}
\label{formulaCoverBoundGeneral}
	\mathcal N(\mathcal H, \varepsilon, \| \cdot \|) \le \prod_{j = 1}^{L} \sup_{F_{M(j)}} \mathcal N_{M(j+1)} \prod_{(u,v,i) \in I_{V}} \sup_{F_{u}} \mathcal N^{u,v,i}_{V} ~.
\end{align}
\end{theorem}

A detailed proof will be given in Section \ref{proofCoverBoundGeneral}.

As vines are chain-like neural networks, we can further obtain an upper bound for $\sup_{F_{u}} \mathcal N^{u,v,i}_{V}$ via a lemma slightly modified from \cite{bartlett2017spectrally}. The lemma is summarised as follows.

\begin{lemma}[Covering Bound for Chain-like Deep Neural Network; cf. \cite{bartlett2017spectrally}, Lemma A.7]
\label{coverChainBoundGeneral}
Suppose there are $L$ weight matrices in a chain-like neural network. Let $(\varepsilon_{1}, \ldots, \varepsilon_{L})$ be given. Suppose the $L$ weight matrices $(A_{1}, \ldots, A_{L})$ lies in $\mathcal B_{1} \times \ldots \times \mathcal B_{L}$, where $\mathcal B_{i}$ is a ball centered at $0$ with the radius of $s_{i}$, i.e., $\mathcal B_{i} = \{A_{i}: \| A_{i} \| \le s_{i}\}$. Furthermore, suppose the input data matrix $X$ is restricted in a ball centred at $0$ with the radius of $B$, i.e., $\| X \| \le B$. Suppose $F$ is a hypothesis function computed by the neural network. If we define:
\begin{equation}
	\mathcal H = \{ F(X): A_{i} \in \mathcal B_{i} \} ~,
\end{equation}
where $i = 1, \ldots, L$ and $t \in \{1, \ldots, L^{u,v,s}\}$. Let $\varepsilon = \sum_{j = 1}^{L}\varepsilon_{j}\rho_{j}\prod_{l = j+1}^{L}\rho_{l}s_{l}$. Then we have the following inequality:
\begin{align}
\label{formulaCoverChainBoundGeneral}
	\mathcal N(\mathcal H, \varepsilon, \| \cdot \|) \le  \prod_{i=1}^{L} \sup_{\mathbf A_{i-1} \in \bm{\mathcal B}_{i-1}} \mathcal N_{i} ~, 
\end{align}
where $\mathbf A_{i-1} = (A_{1}, \ldots, A_{i-1})$, $\bm{\mathcal B}_{i-1} = \mathcal B_{1} \times \ldots \times \mathcal B_{i-1}$, and
\begin{equation}
	\mathcal N_{i}  = \mathcal N \left( \left\{ A_{i}F_{\mathbf A_{i-1}}(X): A_{i} \in \mathcal B_{i} \right\} \varepsilon_{i}, \| \cdot \| \right) ~.
\end{equation}
\end{lemma}

\begin{remark}
	The mapping induced by a chain-like neural network can be formularized as the composition of a series of affine/nonlinear transformations. The proof of Lemma \ref{coverChainBoundGeneral} thus can decompose the covering bound for a chain-like neural network into the product of the covering bounds for all layers (see a detailed proof in \cite{bartlett2017spectrally}). However, residual connections introduce paralleling structures into neural networks. Therefore, the computed mapping cannot be directly expressed as a series of compositions of affine/nonlinear transformations. Instead, to approach a covering bound for the whole network, we are facing many additions of function spaces (see, eq. (\ref{hypothesisFunctionJ})), where the former results cannot be straightly applied. To address this issue, we provide a novel proof collected in Section \ref{proofCoverBoundGeneral}.
\end{remark}

Contrary to the different proofs, the result for deep neural networks with residual connections share similarities with the one for the chain-like network (see, respectively, eq. (\ref{formulaCoverBoundGeneral}) and eq. (\ref{formulaCoverChainBoundGeneral})). The similarities lead to the property summarised as follows.

\begin{quotation}
{\it The influences on the hypothesis complexity of weight matrices are in the same way, no matter whether they are in the stem or the vines. Specifically, adding an identity vine could not affect the hypothesis complexity of the deep neural network.}
\end{quotation}

As indicated by eq. (\ref{formulaCoverChainBoundGeneral}) in Lemma \ref{coverChainBoundGeneral}, the covering number of the hypothesis computed by a chain-like neural network (including the stem and all the vines) is upper bounded by the product of the covering number of all single layers. Specifically, the contribution of the stem on the covering bound is the product of a series of covering numbers, i.e., $\prod_{j = 1}^{L} \sup_{F_{M(j)}} \mathcal N_{M(j+1)}$. In the meantime, applying eq. (\ref{formulaCoverChainBoundGeneral}) in Lemma \ref{coverChainBoundGeneral}, the contribution $\sup_{F_{u}} \mathcal N^{u,v,i}_{V}$ of the vine $V(u,v,i)$ can also be decomposed as the product of a series of covering numbers. Apparently, the contributions respectively by the weight matrices in the stem and the ones in the vines have similar formulations. This result gives an insight that residuals would not undermine the generalization capability of deep neural networks. Also, if a vine $V(u,v,i)$ is an identity mapping, the term in eq. (\ref{formulaCoverBoundGeneral}) that relates to it is definitely $1$, i.e., $\mathcal N^{u,v,i}_{V} = 1$. This is because there is no parameter to tune in an identity vine. This result gives an insight that adding an identity vine to a neural network would not affect the hypothesis complexity.

However, it is worth noting that the vines could influence the part of the stem in the covering bound, i.e., $\mathcal N_{M(j+1)}$ in eq. (\ref{formulaCoverBoundGeneral}). The mechanism of the cross-influence between the stem and the vines is an open problem.

\subsection{Covering Bound for ResNet}
\label{subSecCovBoundResNet}

As an example, we analyze the generalization capability of the $34$-layer ResNet. Analysis of other deep neural networks under the stem-vine framework is similar. For the convenience, we give a detailed illustration of the $34$-layer ResNet under the stem-vine framework in Figure \ref{ResNetStemVineFramework}.

\begin{figure}
\begin{center}
\includegraphics[width=2.9in]{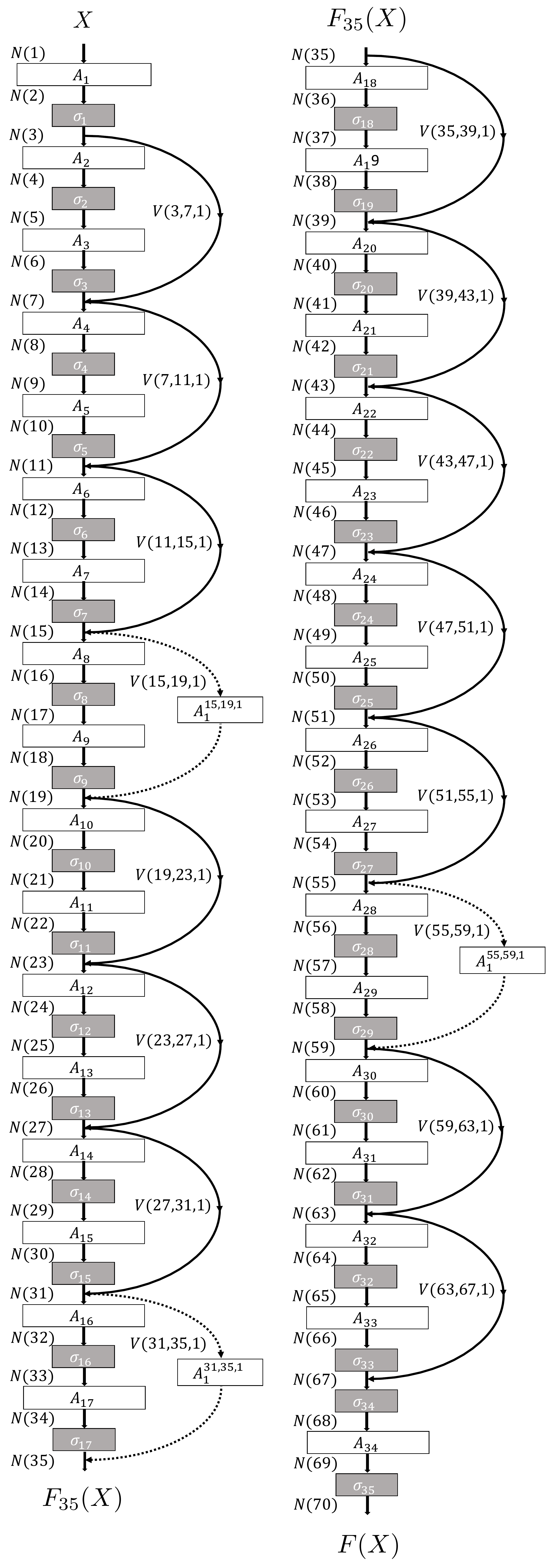}%
\caption{The $34$-layer ResNet under the Stem-Vine Framework.}
\label{ResNetStemVineFramework}
\end{center}
\end{figure}

There are one $34$-layer stem and $16$ vines in the $34$-layer ResNet. Each layer in the stem contains one weight matrix and several Lipschitz-continuous nonlinearities. For most layers with over one nonlinearity, the multiple nonlinearities are connected one by one directly; we merge the nonlinearities as one single nonlinearity. However, the vine links the stem at a vertex between two nonlinearities after the $33$-th weight matrix, and thus we cannot merge the two nonlinearities. Hence, the stem of ResNet can be expressed as follows:
\begin{equation}
	S_{res} = (A_{1}, \sigma_{1}, \ldots, A_{33}, \sigma_{33}, \sigma_{34}, A_{34}, \sigma_{35}) ~.
\end{equation}
From the vertex that receives the input data to the vertex that outputs classification functions, there are $34 + 35 + 1 = 70$ vertexes ($34$ is the number of weight matrices and $35$ is the number of nonlinearities). We denote them as $N(1)$ to $N(70)$. Additionally, we assume the norm of the the weight matrix $A_{i}$ has an upper bound $s_{i}$, i.e., $\| A_{i} \|_{\sigma} \le s_{i}$, while the Lipschitz constant of the nonlinearity $\sigma_{i}$ is denoted as $b_{i}$.

Under the stem-vine framework, the $16$ vines in ResNet are respectively denoted as $V(3, 7, 1), V(7, 11, 1), \ldots, V(63, 67, 1)$. Among these $16$ vines, there are $3$ vines, $V(15, 19, 1)$, $V(31, 35, 1)$, and $V(55, 59, 1)$, that respectively contains one weight matrix, while all others are identity mappings. Let's denote the weight matrices in the vines $V(15, 19, 1)$, $V(31, 35, 1)$, and $V(55, 59, 1)$ respectively as $A^{15,19, 1}_{1}$, $A^{31, 35, 1}_{1}$, and $A^{55, 59, 1}_{1}$. Suppose the norms of $A^{15,19, 1}_{1}$, $A^{31, 35, 1}_{1}$, and $A^{55, 59, 1}_{1}$ are respectively upper bounded by $s^{15,19, 1}_{1}$, $s^{31, 35, 1}_{1}$, and $s^{55, 59, 1}_{1}$. Denote the reference matrices that correspond to weight matrices $(A_{1}, \ldots, A_{34})$ as $(M_{1}, \ldots, M_{34})$. Suppose the distance between each weight matrix $A_{i}$ and the corresponding reference matrix $M_{i}$ is upper bounded by $b_{i}$, i.e., $\| A_{i}^{T} - M_{i}^{T} \| \le b_{i}$. Similarly, suppose there are reference matrices $M_{1}^{s, t, 1}, ~ (s, t) \in \{ (15, 19), (31, 35), (55, 59) \}$ respectively for weight matrices $A^{s, t, 1}_{1}$, and the distance between $A^{s, t}_{1}$ and $M^{s, t, 1}_{1}$ is upper bounded by $b^{s, t, 1}_{1}$, i.e., $\| (A_{i}^{s, t, 1})^{T} - (M_{i}^{s, t, 1})^{T} \| \le b_{1}^{s, t, 1}$. We then have the following lemma.

\begin{lemma}[Covering Number Bound for ResNet]
\label{covBoundResNet}
	For a ResNet $R$ satisfies all conditions above, suppose the hypothesis space is $\mathcal H_{R}$. Then, we have
	\begin{align}
	\label{formulaCovBoundResNet}
		\log \mathcal N(\mathcal H_{R}, \varepsilon, \|\cdot\|) \le & \sum_{u \in \{ 15, 31, 55 \}} \frac{(b^{u,u+4,1}_{1})^{2} \| F_{u}(X^{T})^{T} \|_{2}^{2}}{\varepsilon_{u,u+4,1}^{2}} \log(2W^{2}) \nonumber\\
		& + \sum_{j=1}^{34} \frac{b_{j}^{2} \| F_{2j-1}(X^{T})^{T} \|_{2}^{2}}{\varepsilon_{2j+1}^{2}} \log(2W^{2}) \nonumber\\
		& + \frac{b_{34}^{2} \| F_{68}(X^{T})^{T} \|_{2}^{2}}{\varepsilon_{70}^{2}} \log(2W^{2}) ~,
	\end{align}
	where $\mathcal N(\mathcal H_{R}, \varepsilon, \|\cdot\|)$ is the $\varepsilon$-covering number of $\mathcal H_{R}$. When $j = 1, \ldots, 16$,
	\begin{align}
		\| F_{4j+1} (X) \|_{2}^{2} \le & \| X \|^{2} \rho_{1}^{2} s_{1}^{2} \rho_{2j}^{2} s_{2j}^{2} \prod_{\substack{1 \le i \le j-1 \\ i \notin \{4, 8, 14\}}} \left( \rho_{2i}^{2} s_{2i}^{2} \rho_{2i+1}^{2} s_{2i+1}^{2} + 1 \right) \nonumber\\
		& \prod_{\substack{1 \le i \le j-1 \\ i \in \{4, 8, 14\}}} \left[ \rho_{2i}^{2} s_{2i}^{2} \rho_{2i+1}^{2} s_{2i+1}^{2} + (s^{4i-1, 4i+3, 1}_{1})^{2} \right] ~,
	\end{align}
	and
	\begin{align}
		\| F_{4j+3} (X) \|^{2}_{2} \le & \| X \|^{2} \rho_{1}^{2} s_{1}^{2} \prod_{\substack{1 \le i \le j \\ i \notin \{4, 8, 14\}}} \left( \rho_{2i}^{2} s_{2i}^{2} \rho_{2i+1}^{2} s_{2i+1}^{2} + 1 \right) \nonumber\\
		& \prod_{\substack{1 \le i \le j \\ i \in \{4, 8, 14\}}} \left[ \rho_{2i}^{2} s_{2i}^{2} \rho_{2i+1}^{2} s_{2i+1}^{2} + (s^{4i-1, 4i+3, 1}_{1})^{2} \right] ~,
	\end{align}
	and specifically,
	\begin{align}
%	\label{normResNet}
		\| F_{68}(X^{T})^{T} \|_{2}^{2} \le & \| X \|^{2} \rho_{1}^{2} s_{1}^{2} \rho_{34}^{2} \prod_{\substack{1 \le i \le 16 \\ i \notin \{4, 8, 14\}}} \left( \rho_{2i}^{2} s_{2i}^{2} \rho_{2i+1}^{2} s_{2i+1}^{2} + 1 \right) \nonumber\\
		& \prod_{\substack{1 \le i \le 16 \\ i \in \{4, 8, 14\}}} \left[ \rho_{2i}^{2} s_{2i}^{2} \rho_{2i+1}^{2} s_{2i+1}^{2} + (s^{4i-1, 4i+3, 1}_{1})^{2} \right] ~.
	\end{align}
	Also, when $j = 1, \ldots, 16$,
	\begin{align}
		\varepsilon_{4j+1} = & (1 + s_{1})\rho_{1}(1 + s_{2j})\rho_{2j} \prod_{\substack{1 \le i \le j-1 \\ i \notin \{ 4, 8, 14 \}}} \left[ \rho_{2i}(s_{2i} + 1)\rho_{2i+1}(s_{2i+1} + 1) + 1 \right] \nonumber\\
		& \prod_{\substack{1 \le i \le j-1 \\ i \in \{ 4, 8, 14 \}}} \left[ \rho_{2i}(s_{2i} + 1)\rho_{2i+1}(s_{2i+1} + 1) + 1 + s^{4i-1, 4i+3, 1}_{1} \right] ~,
	\end{align}
	and
	\begin{align}
		\varepsilon_{4j+3} = & (1 + s_{1})\rho_{1} \prod_{\substack{1 \le i \le j \\ i \notin \{ 4, 8, 14 \}}} \left[ \rho_{2i}(s_{2i} + 1)\rho_{2i+1}(s_{2i+1} + 1) + 1 \right] \nonumber\\
		& \prod_{\substack{1 \le i \le j \\ i \in \{ 4, 8, 14 \}}} \left[ \rho_{2i}(s_{2i} + 1)\rho_{2i+1}(s_{2i+1} + 1) + 1 + s^{4i-1, 4i+3, 1}_{1} \right] ~,
	\end{align}	
	and for $u = 15, 31, 55$,
	\begin{align}
%	\label{formulaEpsilonVine}
		\varepsilon_{u, u+4, 1} = \varepsilon_{u} \left(1+s^{u,u+4,1}_{1}\right) ~.
	\end{align}
	In above equations/inequalities,
	\begin{align}
%	\label{formulaAlpha}
	\bar \alpha = & (s_{1}+1)\rho_{1}\rho_{34}(s_{34}+1)\rho_{35} \prod_{\substack{1 \le i \le 16 \\ i \notin \{ 4,8,14 \}}} \left[ \rho_{2i}(s_{2i} + 1)\rho_{2i+1}(s_{2i+1} + 1) + 1 \right]\nonumber\\
	& \prod_{i \in \{ 4,8,14 \}} \left[ \rho_{2i}(s_{2i} + 1)\rho_{2i+1}(s_{2i+1} + 1) + s^{4i-1,4i+3,1}_{1} + 1 \right] ~.
	\end{align}
%	and
%	\begin{equation}
%		(*) = \rho_{2i}(s_{2i} + 1)\rho_{2i+1}(s_{2i+1} + 1) ~.
%	\end{equation}
\end{lemma}

A detailed proof is omitted and will be given in Section \ref{proofCoverBoundGeneral}.

\subsection{Generalization Bound for ResNet}
\label{subGeneralizationBoundResNet}

Lemmas \ref{RadBound} and \ref{covNumBound} guarantee that when the covering number of a hypothesis space is upper bounded, the corresponding generalization error is upper bounded. Therefore, combining the covering bound for ResNet given by Lemma \ref{covBoundResNet}, a generalization bound for ResNet is straight-forward. In this subsection, the generalization bound is summarized as Theorem \ref{generalizartionBoundResNet}.

For the brevity, we rewrite the radius $\varepsilon_{2j+1}$ and $\varepsilon_{u,u+4,1}$ as follows:
\begin{gather}
	\varepsilon_{2j+1} = \hat{\varepsilon}_{2j+1}~,\\
	\varepsilon_{u,u+4,1} = \hat{\varepsilon}_{u,u+4,1} \varepsilon ~.
\end{gather}
Additionally, we rewrite eq. (\ref{formulaCovBoundResNet}) of Lemma \ref{covBoundResNet} as the following inequality:
\begin{equation}
\label{reFormulaCovBoundResNet}
	\log \mathcal N(\mathcal H, \varepsilon, \|\cdot\|) \le \frac{R}{\varepsilon^{2}} ~,
\end{equation}
where
\begin{align}
	\label{defR}
		R = & \sum_{u \in \{ 15, 31, 55 \}} \frac{(b^{u,u+4,1}_{1})^{2} \| F_{u}(X^{T})^{T} \|_{2}^{2}}{\hat\varepsilon_{u,u+4,1}^{2}} \log(2W^{2}) \nonumber\\
	& + \sum_{j=1}^{33} \frac{b_{j}^{2} \| F_{2j-1}(X^{T})^{T} \|_{2}^{2}}{\hat\varepsilon_{2j+1}^{2}} \log(2W^{2}) \nonumber\\
	& + \frac{b_{34}^{2} \| F_{68}(X^{T})^{T} \|_{2}^{2}}{\hat\varepsilon_{70}^{2}} \log(2W^{2}) ~,
	\end{align}
Then, we can obtain the following theorem.

\begin{theorem}[Generalization Bound for ResNet]
\label{generalizartionBoundResNet}

Suppose a ResNet satisfies all conditions in Lemma \ref{covBoundResNet}. Suppose a given series of examples $(x_{1}, y_{1}), \ldots, (x_{n}, y_{n})$ are arbitrary independent and identically distributed (iid) variables drawn from any distribution over $\mathcal R^{n_{0}} \times \{ 1, \ldots, n_{L} \}$. Suppose hypothesis function $F_{\mathcal A}: \mathbb R^{n_{0}} \to \mathbb R^{n_{L}}$ is computed by a ResNet with weight matrices $\mathcal A = (A_{1}, \ldots, A_{34}, A^{15,19,1}_{1}, A^{31,35,1}_{1}, A^{55,59,1}_{1})$. Then for any margin $\lambda > 0$ and any real $\delta \in (0, 1)$, with probability at least $1 - \delta$, we have the following inequality:
\begin{align}
\label{generalizationBoundResNet}
	\Pr\{ \arg\max_{i} F(x)_{i} \ne y \} \le \hat{\mathcal R}_{\lambda}(F) + \frac{8}{n^{\frac{3}{2}}} + \frac{36}{n} \sqrt{R} \log n + 3 \sqrt{\frac{\log(1/\delta)}{2n}} ~,
\end{align}
where $R$ is defined as eq. (\ref{defR}).
\end{theorem}

A proof is omitted here and will be given in Section \ref{sec:generalizartionBoundResNet}.

Indicated by Theorem \ref{generalizartionBoundResNet}, the generalization bound of ResNet relies on its covering bound. Specifically, when the sample size $n$ and the probability $\delta$ are fixed, the generalization error satisfies that
\begin{align}
\label{insightR}
	\Pr\{ \arg\max_{i} F(x)_{i} \ne y \} - \hat{\mathcal R}_{\lambda}(F) = \mathcal O\left(\sqrt{R}\right) ~,
\end{align}
where $R$ expresses the magnitude of the covering number ($R/\varepsilon^{2}$ is an $\varepsilon$-covering bound).
Combining the property generally for any neural network under the stem-vine framework, eq. (\ref{insightR}) gives two insights about the effects of residual connections on the generalization capability of neural networks: (1) The influences of weight matrices on the generalization capability are invariant, no matter where they are (either in the stem or in the vines); (2) Adding an identity vine could not affect the generalization. These results give an theoretical explanation of why ResNet has equivalently good generalization capability as the chain-like neural networks.

As indicated by eq. (\ref{generalizationBoundResNet}), the expected risk (or, equivalently, the expectation of the test error) of ResNet equals the sum of the empirical risk (or, equivalently, the training error) and the generalization error. In the meantime, residual connections significantly reduce the training error of the neural network in many tasks. Our results therefore theoretically explain why ResNet has a significantly lower test error in these tasks.

\subsection{Practical Implementation}
\label{implication}

\begin{figure}
\centering
\subfigure[$0$, $128$, $7.37\%$]{\includegraphics[width = 0.4\linewidth]{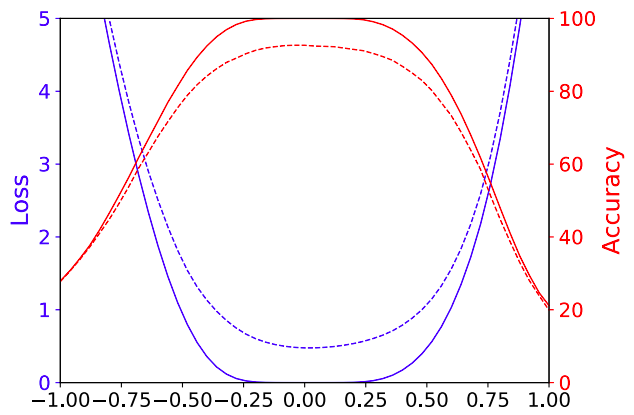}} \qquad
\subfigure[$5 \times 10^{-4}$, $128$, $6.00\%$]{\includegraphics[width = 0.41\linewidth]{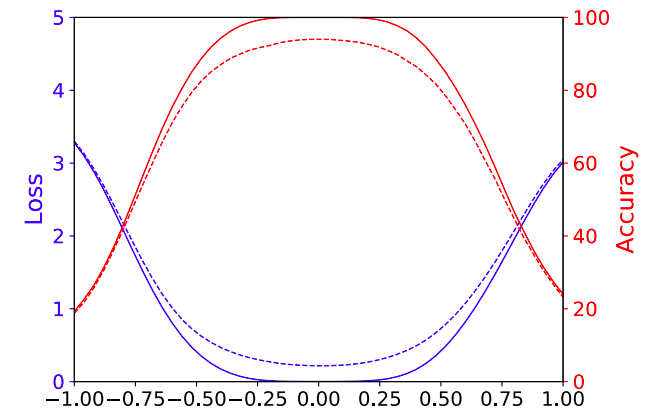}} \\
\subfigure[$0$, $128$, $7.37\%$]{\includegraphics[width = 0.36\linewidth]{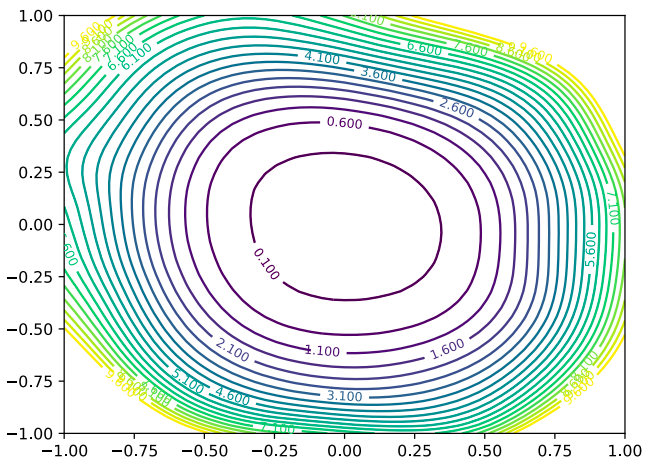}} \qquad
\subfigure[$5 \times 10^{-4}$, $128$, $6.00\%$]{\includegraphics[width = 0.389\linewidth]{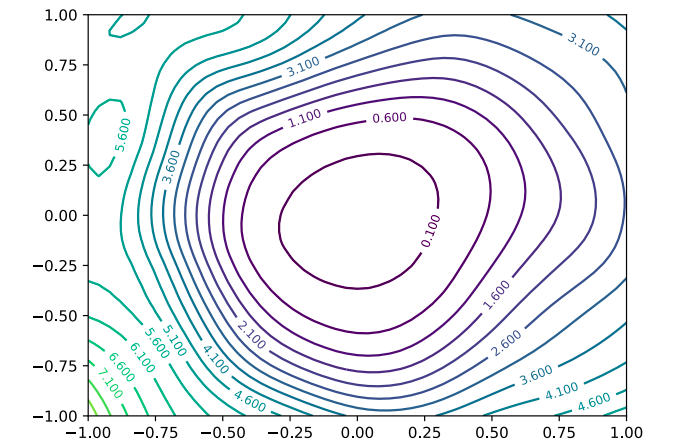}} \\
\subfigure[$0$, $8192$, $11.07\%$]{\includegraphics[width = 0.41\linewidth]{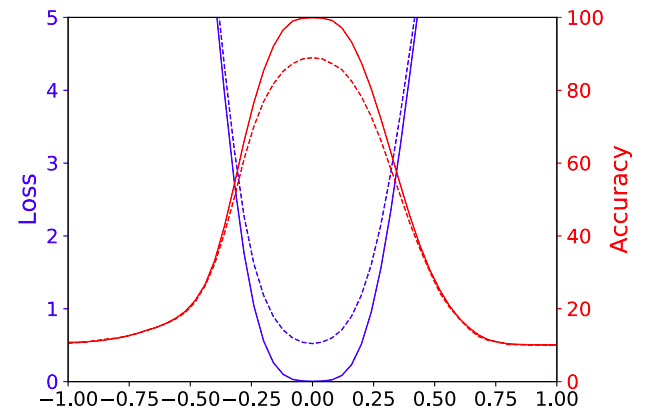}} \qquad
\subfigure[$5 \times 10^{-4}$, $8192$, $10.19\%$]{\includegraphics[width = 0.4\linewidth]{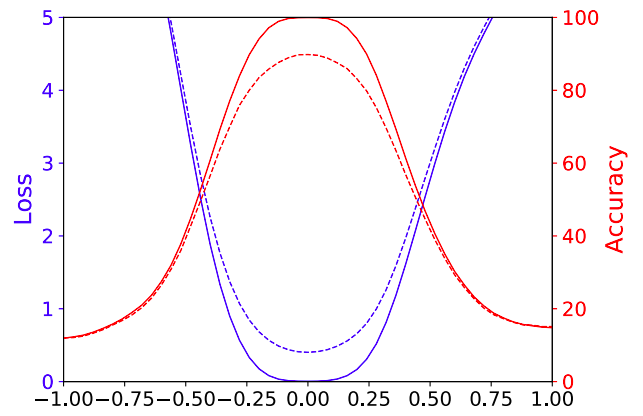}}\\
\subfigure[$0$, $8192$, $11.07\%$]{\includegraphics[width = 0.385\linewidth]{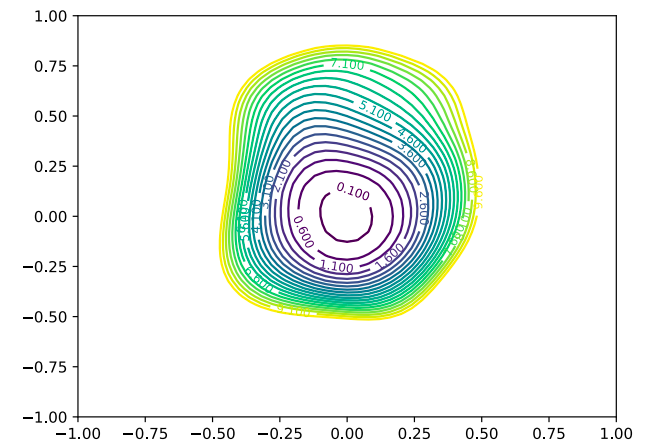}} \qquad
\subfigure[$5 \times 10^{-4}$, $8192$, $10.19\%$]{\includegraphics[width = 0.385\linewidth]{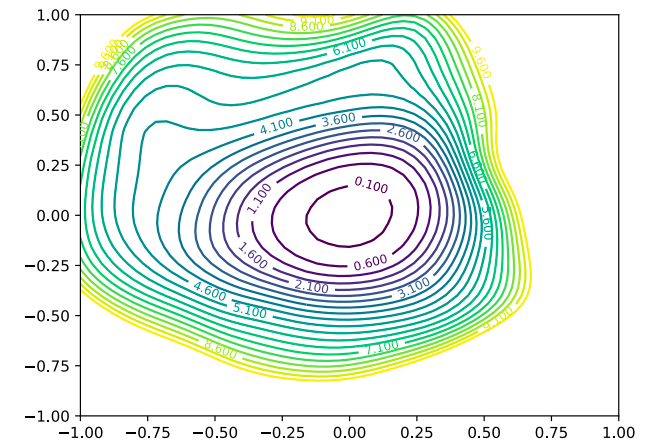}} \\
\caption{Illustrations of the 1D and 2D visualization of the loss surface around the solutions obtained with different weight decay and batch size. The numbers in the title of each subfigure is respectively the parameter of weight decay, batch size, and test error. The data and figures are originally presented in \cite{li2018visualizing}.}
\label{fig:weight_decay}
\end{figure}

Besides the sample size $N$, our generalization bound (eq. (\ref{generalizationBoundResNet})) has a positive correlation with the norms of all the weight matrices. Specifically, weight matrices with higher norms lead to a higher generalization bound of the neural network, and therefore leads to a worse generalization ability. This feature induces a practical implementation which justifies the standard of technique weight decay.

Weight decay can be dated back to a paper by Krogh and Hertz \cite{krogh1992simple} and is widely used in training deep neural networks. It uses the $L_{2}$ norm of all the weights as a regularization term to control the magnitude of the norms of the weights not to increase too much:

\begin{remark}
	The technique of weight decay can improve the generalization ability of deep neural networks. It refers to adding the $L_{2}$ norm of the weights $w = (w_{1}, \ldots, w_{D})$ to the objective function as a regularization term:
\begin{equation*}
	\mathcal L'(w) = \mathcal L(w) + \frac{1}{2} \lambda \sum_{i = 1}^{D} w_{i}^{2} ~,
\end{equation*}
where $\lambda$ is a tuneable parameter, $\mathcal L(w)$ is the original objective function, and $\mathcal L'(w)$ is the objective function with weight decay.
\end{remark}

The term $\frac{1}{2} \lambda \sum_{i = 1}^{D} w_{i}^{2}$ can be easily re-expressed by the $L_{2}$ norms of all the weight matrices. Therefore, using weight decay can control the magnitude of the norms of all the weights matrices not to increase too much. Also, our generalization bound (eq. (\ref{generalizationBoundResNet})) provides a positive correlation between the generalization bound and the norms of all the weight matrices. Thus, our work gives a justification for why weight decay leads to a better generalization ability.

A recent systematic experiment conducted by Li et al. studies the influence of weight decay on the loss surface of the deep neural networks \cite{li2018visualizing}. It trains a 9-layer VGGNet \cite{chen2018deeplab} on the dataset CIFAR-10 \cite{krizhevsky2009learning} by employing stochastic gradient descent with batch sizes of $128$ ($0.26\%$ of the training set of CIFAR-10) and $8192$ ($16.28\%$ of the training set of CIFAR-10). The results demonstrate that by employing weight decay, SGD can find flatter minima\footnote{The flatness (or equivalently sharpness) of the loss surface around the minima is considered as an important index expressing the generalization ability. However, the mechanism still remains elusive. For more details, please refers to \cite{keskar2016large} and \cite{dinh2017sharp}.} of the loss surface with lower test errors as shown in fig. \ref{fig:weight_decay} (original presented as \cite{li2018visualizing}, p. 6, fig. 3). Other technical advances and empirical analysis include \cite{galloway2018adversarial, zhang2018three, %ayinde2018deep, 
chen2018closing, epark2019bayesian}.

\section{Proofs}
\label{sec:proof}

This appendix collects various proofs omitted from Section \ref{generalizationBound}. We first give a proof of the covering bound for an affine transformation induced by a single weight matrix. It is the foundation of the other proofs. Then, we provide a proof of the covering bound for deep neural networks under the stem-vine framework (Theorem \ref{coverBoundGeneral}). Furthermore, we present a proof of the covering bound for ResNet (Lemma \ref{covBoundResNet}). Eventually, we provide a proof of the generalization bound for ResNet (Theorem \ref{generalizartionBoundResNet}).

\subsection{Proof of the Covering Bound for the Hypothesis Space of a Single Weight Matrix}

In this subsection, we provide an upper bound for the covering number of the hypothesis space induced by a single weight matrix $A$. This covering bound relies on {\it Maurey sparsification lemma} \cite{pisier1981remarques} and has been introduced in machine learning by previous works (see, e.g.,\cite{zhang2002covering,bartlett2017spectrally}).

Suppose a data matrix $X$ is the input of a weight matrix $A$. All possible values of the output $XA$ constitute a space. We use the following lemma to express the complexity of all $XA$ via the covering number.
\begin{lemma}[Bartlett et al.; see \cite{bartlett2017spectrally}, Lemma 3.2]
\label{matrixCover}
	Let conjugate exponents $(p, q)$ and $(r, s)$ be given with $p \le 2$, as well as positive reals $(a, b, \varepsilon)$ and positive integer $m$. Let matrix $X \in \mathbb R^{n \times d}$ be given with $\| X \|_{p} \le b$. Let $\mathcal H_{A}$ denote the family of matrices obtained by evaluating $X$ with all choices of matrix $A$:
	\begin{equation}
		\mathcal H_{A} \triangleq \left\{ XA | A \in \mathbb R^{d \times m}, \|A\|_{q, s} \le a \right\} ~.
	\end{equation}
	Then
	\begin{equation}
		\log \mathcal N \left( \mathcal H_{A}, \varepsilon, \| \cdot \|_{2} \right) \le \ceil*{\frac{a^{2}b^{2}m^{2/r}}{\varepsilon^{2}}} \log(2dm) ~.
	\end{equation}
\end{lemma}

\subsection{Covering Bound for the Hypothesis Space of Chain-like Neural Network}

This subsection considers the upper bound for the covering number of the hypothesis space induced by the stem of a deep neural network. Intuitively, following the stem from the first vertex $N(1)$ to the last one $N(L)$, every weight matrices and nonlinearities increase the complexity of the hypothesis space that could be computed by the stem. Following this intuition, we use an induction method to approach the upper bound. The result is summarized as Lemma \ref{coverChainBoundGeneral}. This lemma is originally given in the work by Bartlett et al. \cite{bartlett2017spectrally}. Here to make this work complete, we recall the main part of the proof but omit the part for $\varepsilon$.

\begin{proof}[Proof of Lemma \ref{coverChainBoundGeneral}]
We use an induction procedure to prove the lemma.
%\begin{enumerate}
%\item

(1) The covering number of the hypothesis space computed by the first weight matrix $A_{1}$ can be straightly upper bounded by Lemma \ref{matrixCover}.
%\item

(2) The vertex after the $j$-th nonlinearity is $N(2j+1)$. Suppose $\mathcal W_{2j+1}$ is an $\varepsilon$-cover of the hypothesis space $\mathcal H_{2j+1}$ induced by the output hypotheses in the vertex $N(2j+1)$. Suppose there is a weight matrix $A_{j+1}$ directly follows the vertex $N(2j+1)$. We then analyze the contribution of the weight matrix $A_{j+1}$. Assume that there exists an upper bound $s_{j+1}$ of the norm of $A_{j+1}$. For any $F_{2j+1}(X) \in \mathcal H_{2j+1}$, there exists a $W(X) \in \mathcal W_{2j+1}$ such that
\begin{equation}
\label{ithCover}
	\| F_{2j+1}(X) - W(X) \| \le \varepsilon_{2j+1} ~.
\end{equation}
Lemma \ref{matrixCover} guarantees that for any $W(X) \in \mathcal W_{2j+1}$ there exists an $\varepsilon_{2j+1}$-cover $\mathcal W_{2j+2}(W)$ for the function space $\{ W(X)A_{j+1}: W(X) \in \mathcal W_{2j+1}, \| A_{j+1} \| \le s_{j+1}\}$, i.e., for any $W'(X) \in \hat{\mathcal H}_{2j+1}$, there exists a $V(X) \in \{ W(X)A_{j+1}: W(X) \in \mathcal W_{2j+1}, \| A_{j+1} \| \le s_{j+1}\}$ such that
\begin{equation}
\label{ithMatrixCover}
	\| W'(X) - V(X) \| \le \varepsilon_{2j+1} ~.
\end{equation}
As for any $F_{2j+1}'(X) \in \mathcal H_{2j+2} \triangleq \{ F_{2j+1}(X)A_{j+1}: F_{2j+1}(X) \in \mathcal H_{2j+1}, \| A_{j+1} \| \le c\}$, there is a $F_{2j+1}(X) \in \mathcal H_{2j+1}$ such that
\begin{equation}
\label{matrixTrans}
	F_{2j+1}'(X) = F_{2j+1}(X) A_{j+1} ~.
\end{equation}
Thus, applying eqs. (\ref{ithCover}), (\ref{ithMatrixCover}), and (\ref{matrixTrans}), we get the following inequality
\begin{align}
\label{radiusInductionMatrix}
	& \| F_{2j+1}'(X) - V(X) \| \nonumber\\
	= & \| F_{2j+1}(X)A_{j+1} - V(X) \| \nonumber\\
	= & \| F_{2j+1}(X)A_{j+1} - W(X)A_{j+1} + W(X)A_{j+1} - V(X) \| \nonumber\\
	\le & \| F_{2j+1}(X)A_{j+1} - W(X)A_{j+1} \|  + \| W(X)A_{j+1} - V(X) \| \nonumber\\
	\le & \| F_{2j+1}(X) - W(X) \| \| A_{j+1} \| + \varepsilon_{2j+1} \nonumber\\
	\le & s_{j+1} \varepsilon_{2j+1} + \varepsilon_{2j+1} \nonumber\\
	= & (s_{j+1} + 1) \varepsilon_{2j+1} ~.
\end{align}
Therefore, $\bigcup_{W \in \mathcal W_{2j+1}} \mathcal W_{2j+2}(W)$ is a $(s_{j+1} + 1) \varepsilon_{2j+1}$-cover of $\mathcal H_{2j+2}$. Let's denote $(s_{j+1} + 1) \varepsilon_{2j+1}$ as $\varepsilon_{2j+2}$. Apparently,
\begin{align}
	& \mathcal N(\mathcal H_{2j+2}, \varepsilon_{2j+2}, \| \cdot \|) \nonumber\\
	\le & \left|\bigcup_{W \in \mathcal W_{2j+1}} \mathcal W_{2j+2}(W) \right| \nonumber\\
	\le & \left|\mathcal W_{2j+1}\right| \cdot \sup_{W \in \mathcal W_{2j+1}} \left|\mathcal W_{2j+2}(W)\right| \nonumber\\
	\le & \mathcal N(\mathcal H_{2j+1}, \varepsilon_{2j+1}, \| \cdot \|) \nonumber\\
	& \sup_{\substack{(A_{1}, \ldots, A_{j})\\ \forall j \le j,~A_{i} \in \mathcal B_{i}}} \mathcal N \left( %\heartsuit 
	\left\{ A_{j+1}F_{2j+1}(X): A_{j+1} \in \mathcal B_{j+1} \right\}, \varepsilon_{2j+1}, \| \cdot \|_{2j+1}\right) ~.
\end{align}
%where
%\begin{equation*}
	%\heartsuit 
%	(**) = \left\{ A_{j+1}F_{2j+1}(X): A_{j+1} \in \mathcal B_{j+1} \right\} ~.
%\end{equation*}
Thus, $\mathcal N(\mathcal W_{2j+1}, \varepsilon_{2j+1}, \| \cdot \|) \cdot \mathcal N(\mathcal W_{2j+2}, \varepsilon_{2j+2}, \| \cdot \|)$ is an upper bound for the $\varepsilon_{2j+2}$-covering number of the hypotheses space $\mathcal H_{i+1}$.

%\item
(3) The vertex after the $j$-th weight matrix is $N(2j-1)$. Suppose $\mathcal W_{2j-1}$ is an $\varepsilon_{2j-1}$-cover of the hypothesis space $\mathcal H_{2j-1}$ induced by the output hypotheses in the vertex $N(2j-1)$. Suppose there is a nonlinearity $\sigma_{j}$ directly follows the vertex $N(2j-1)$. We then analyze the contribution of the nonlinearity $\sigma_{j}$. Assume that the nonlinearity $\sigma_{j}$ is $\rho_{j}$-Lipschitz continuous. Apparently, $\sigma_{j}(\mathcal W_{2j-1})$ is a $\rho\varepsilon_{2j-1}$-cover of the hypothesis space $\sigma_{j}(\mathcal H_{2j-1})$. Specifically, for any $F' \in \sigma(\mathcal H_{2j-1})$, there exits a $F \in \mathcal H_{2j-1}$ that $F' = \sigma_{j}(F)$. Since $\mathcal W_{2j-1}$ is an $\varepsilon_{2j-1}$-cover of the hypothesis space $\mathcal H_{2j-1}$, there exists a $W \in \mathcal W_{2j-1}$ such that
\begin{equation}
	\| F - W_{2j-1} \| \le \varepsilon_{2j-1} ~.
\end{equation}
Therefore, we have the following equation
\begin{align}
\label{radiusInductionNonlinearity}
	& \| F' - \sigma_{j}(W_{2j-1}) \| \nonumber\\
	= & \| \sigma_{j}(F) - \sigma_{j}(W_{2j-1}) \| \nonumber\\
	\le & \rho_{j} \| F - W_{2j-1} \| = \rho_{j}\varepsilon_{2j-1} ~.
\end{align}
We thus prove that $\mathcal W_{2j} \triangleq \sigma_{j}(\mathcal W_{2j-1})$ is a $\rho_{j} \varepsilon_{2j-1}$-cover of the hypothesis space $\sigma_{j}(\mathcal H_{2j-1})$. Additionally, the covering number remains the same while applying a nonlinearity to the neural network.
%\end{enumerate}

By analyzing the influence of weight matrices and nonlinearities one by one, we can get eq. (\ref{formulaCoverChainBoundGeneral}). As for $\varepsilon$, the above part indeed gives an constructive method to obtain $\varepsilon$ from all $\varepsilon_{i}$ and $\varepsilon_{u,v,j}$. Here we omit the explicit formulation of $\varepsilon$ in terms of $\varepsilon_{i}$ and $\varepsilon_{u,v,j}$, since it could not benefit our theory.
\end{proof}

\subsection{Covering Bound for the Hypothesis Space of Deep Neural Networks with Residual Connections}
\label{proofCoverBoundGeneral}

In Subsection \ref{subSecCoverBoundGeneral}, we give a covering bound generally for all deep neural networks with residual connections. The result is summarised as Theorem \ref{coverBoundGeneral}. In this subsection, we give a detailed proof of Theorem \ref{coverBoundGeneral}.

\begin{proof}[Proof of Theorem \ref{coverBoundGeneral}]
To approach the covering bound for the deep neural networks with residuals, we first analyze the influence of adding a vine to a deep neural network, and then use an induction method to obtain a covering bound for the whole network.

All vines are connected with the stem at two points that is respectively after a nonlinearity and before a weight matrix. When the input $F_{u}(X)$ of the vine $V(u,v,i)$ is fixed, suppose all the hypothesis functions $F_{V}^{u,v,i}(X)$ computed by the vine $V(u,v,i)$ constitute a hypothesis space $\mathcal H_{V}^{u,v,i}$. As a vine is also a chain-like neural network constructed by stacking a series of weight matrices and nonlinearities, we can straightly apply Lemma \ref{coverChainBoundGeneral} to approach an upper bound for the covering number of the hypothesis space $\mathcal H_{V}^{u,v,i}$. It is worth noting that vines could be identity mappings. This situation is normal in ResNet -- there are $13$ out of all the $16$ vines are identities. For the circumstances that the vines are identities, the hypothesis space computed by the vine only contains one element -- an identity mapping. The covering number of the hypothesis space for the identities are apparently $1$.

Applying Lemmas \ref{matrixCover} and \ref{coverChainBoundGeneral}, there exists an $\varepsilon_{v}$-cover $\mathcal W_{v}$ for the hypothesis space $\mathcal H_{v}$ with a covering number $\mathcal N(\mathcal H_{v}, \varepsilon_{i}, \| \cdot \|)$, as well as an $\varepsilon_{V}^{u,v,i}$-cover $\mathcal W_{V}^{u,v,i}$ for the hypothesis space $\mathcal H_{V}^{u,v,i}$ with a covering number $\mathcal N(\mathcal H_{V}^{u,v,i}, \varepsilon_{i}, \| \cdot \|)$.

The hypotheses computed by the vine $V(u,v,i)$ and the deep neural network without $V(u,v,i)$, i.e., respectively, $F_{v}(X)$ and $F_{V}^{u,v,i}$, are added element-wisely at the vertex $V(v)$. We denote the space constituted by all $F' \triangleq F_{v}(X) + F_{V}^{u,v,i}(X)$ as $\mathcal H'_{v}$.

Let's define a function space as $\mathcal W'_{v} \triangleq \{ W_{S} + W_{V}: W_{S} \in \mathcal W_{v}, W_{V} \in \mathcal W_{V}^{u,v,i} \}$. For any hypothesis $F' \in \mathcal H'_{v}$, there must exist an $F_{S} \in \mathcal H_{v}$ and $F_{V} \in \mathcal H_{V}^{u,v,i}$ such that
\begin{equation}
	F'(X) = F_{S}(X) + F_{V}(X) ~.
\end{equation}
Because $\mathcal W_{v}$ is an $\varepsilon_{v}$-cover of the hypothesis space $\mathcal H_{v}$. For any hypothesis $F_{S} \in \mathcal H_{v}$, there exists an element $W_{ F_{S}}(X) \in \mathcal W_{v}$, such that
\begin{equation}
\label{coverOriginal}
	\| F_{S}(X) - W_{ F_{S}}(X) \| \le \varepsilon_{v} ~.
\end{equation}
Similarly, as $\mathcal W_{V}^{u,v,i}$ is an $\varepsilon_{V}^{u,v,i}$-cover of $\mathcal H_{V}^{u,v,i}$, we can get a similar result. For any hypothesis $F_{V}(X) \in \mathcal H_{V}^{u,v,i}$, there exists an element $W_{F_V}(X) \in \mathcal W_{V}^{u,v,i}$, such that
\begin{equation}
\label{coverVine}
	\| F_{V}(X) - W_{F_V}(X) \| \le \varepsilon_{V}^{u,v,i} ~.
\end{equation}
Therefore, For any hypothesis $F'(X) \in \mathcal H'_{v}$, there exists an element $W(X) \in \mathcal W'$, such that $W(X) = W_{ F_{S}}(X) +W_{F_V}(X)$ satisfying eqs. (\ref{coverOriginal}) and (\ref{coverVine}), and furthermore,
\begin{align}
	& \| F'(X) - W(X) \| \nonumber\\
	= & \| F_{V}(X) + F_{S}(X) - W_{F_V}(X) - W_{ F_{S}}(X) \| \nonumber\\
	= & \| (F_{V}(X) - W_{F_V}(X)) + (F_{S}(X) - W_{ F_{S}}(X)) \| \nonumber\\
	\le & \| F_{V}(X) - W_{F_V}(X) \| + \| F_{S}(X) - W_{ F_{S}}(X) \| \nonumber\\
	\le & \varepsilon_{V}^{u,v,i} + \varepsilon_{v} ~.
\end{align}
Therefore, the function space $\mathcal W'_{v}$ is an $(\varepsilon_{V}^{u,v,i} + \varepsilon_{v})$-cover of the hypothesis space $\mathcal H'_{v}$. An upper bound for the cardinality of the function space $\mathcal W'_{v}$ is given as below (it is also an $\varepsilon_{V}^{u,v,i} + \varepsilon_{v}$-covering number of the hypothesis space $\mathcal H'_{v}$):
\begin{align}
	& \mathcal N(\mathcal H'_{v}, \varepsilon_{V}^{u,v,i} + \varepsilon_{v}, \| \cdot \|) \nonumber\\
	\le & |\mathcal W'_{v}| \le |\mathcal W_{v}| \cdot |\mathcal W_{V}^{u,v,i}| \nonumber\\
	\le & \sup_{F_{v-2}} \mathcal N(\mathcal H_{v}, \varepsilon_{i}, \| \cdot \|) \cdot \sup_{F_{u}} \mathcal N(\mathcal H_{V}^{u,v,i}, \varepsilon_{V}^{u,v,i}, \| \cdot \|) \nonumber\\
	\le & \sup_{F_{v-2}} \mathcal N_{v} \cdot \sup_{F_{u}} \mathcal N_{V}^{u,v,i} ~,
\end{align}
where $\mathcal N_{v}$ and $\mathcal N_{V}^{u,v,i}$ can be obtained from eq. (\ref{formulaCoverChainBoundGeneral}) in Lemma \ref{coverChainBoundGeneral}, as the stem and all the vines are chain-like neural networks.

By adding vines to the stem one by one, we can construct the whole deep neural network. Combining Lemma \ref{coverChainBoundGeneral} for the covering number of $F_{v-1}(X)$ and $F_{u}(X)$, we further get the following inequality:
\begin{align}
%\label{formulaCoverBoundGeneral}
	\mathcal N(\mathcal H, \varepsilon, \| \cdot \|) \le \prod_{j = 1}^{L} \sup_{F_{M(j)}} \mathcal N_{M(j+1)} \prod_{(u,v,i) \in I_{V}} \sup_{F_{u}} \mathcal N^{u,v,i}_{V} ~.
\end{align}
Thus, we prove eq. (\ref{formulaCoverBoundGeneral}) of Theorem \ref{coverBoundGeneral}.

As for $\varepsilon$, the above part indeed gives an constructive method to obtain $\varepsilon$ from all $\varepsilon_{i}$ and $\varepsilon_{u,v,j}$. Here we omit the explicit formulation of $\varepsilon$ in terms of $\varepsilon_{i}$ and $\varepsilon_{u,v,j}$, since it could be extremely complex and does not benefit our theory.
\end{proof}

\subsection{Covering Bound for the Hypothesis Space of ResNet}

In Subsection \ref{subSecCovBoundResNet}, we give a covering bound for ResNet. The result is summarized as Lemma \ref{covBoundResNet}. In this subsection, we give a detailed proof of Lemma \ref{covBoundResNet}.

\begin{proof}[Proof of Lemma \ref{covBoundResNet}]

There are $34$ weight matrices and $35$ nonlinearities in the stem of the $34$-ResNet. Let's denote the weight matrices respectively as $A_{1}$, ... ,$A_{34}$ and denote the nonlinearities respectively as $\sigma_{1}$, ... , $\sigma_{35}$. Apparently, there are $34 + 35 + 1 = 70$ vertexes in the network, where $34$ is the number of weight matrices and $35$ is the number of nonlinearities. We denote them respectively as $N(1)$, ... , $N(70)$. Additionally, there are $16$ vines which are respectively denoted as $V(4i - 1, 4i+3, 1)$, $i = \{1, \ldots, 16\}$, where $4i - 1$ and $4i+3$ are the indexes of the vertexes that the vine connected. Among all the $16$ vines, there are $3$, $V(15, 19, 1)$, $V(31, 35, 1)$, and $V(55, 59, 1)$, respectively contain one weight matrix, while all others are identities mappings. For the vine $V(4i-1, 4i+3, 1)$, $i = 4, 8, 14$, we denote the weight matrix in the vine as $A^{4i-1, 4i+3, 1}_{1}$.

Applying Theorem \ref{coverBoundGeneral}, we straightly get the following inequality:
\begin{align}
\label{coverBoundResNet}
	\log \mathcal N(\mathcal H, \varepsilon, \| \cdot \|) \le \sum_{j=1}^{34} \sup_{F_{2j-1}(X)} \log \mathcal N_{2j+1} + \sum_{(u,v,i) \in I_{V}} \sup_{F_{u}(X)} \log \mathcal N^{u,v,1}_{V} ~,
\end{align}
where $\mathcal N_{2j+1}$ is the covering number of the hypothesis space constituted by all outputs $F_{2j+1}(X)$ at the vertex $N(2j+1)$ when the input $F_{2j-1}(X)$ of the vertex $N(2j-1)$ is fixed, $\mathcal N^{u,v,1}_{V}$ is the covering number of the hypothesis space constituted by all outputs $F_{V}^{u,v,i}(X)$ of the vine $V(u,v,1)$ when the input $F_{v}(X)$ is fixed, and $I_{V}$ is the index set $\{(4i-1, 4i+3, 1), i = 1, \ldots, 16\}$.

Applying Lemma \ref{matrixCover}, we can further get an upper bound for the $\varepsilon_{2j+1}$-covering number $\mathcal N_{2j+1}$. The bound is expressed as the following inequality:
\begin{equation}
\label{coverBoundStemResNet}
	\log \mathcal N_{2j+1} \le \frac{b_{j+1}^{2} \| F_{2j+1}(X^{T})^{T} \|_{2}^{2}}{\varepsilon_{2j+1}^{2}} \log(2W^{2}) ~,%\varepsilon
\end{equation}
where $W$ is the maximum dimension among all features through the ResNet, i.e., $W = \max_{i} n_{i}$, $i = 0, 1, \ldots, L$. Also, we can decompose $\| F_{2j+1}(X^{T})^{T} \|_{2}^{2}$ and utilize an induction method to obtain an upper bound for it.
%\begin{enumerate}
%\item

(1) If there is no vine connected with the stem at the vertex $N(2j-1)$, we have the following inequality:
\begin{align}
\label{inductionNoVineResNet}
	& \| F_{2j+1}(X^{T})^{T} \|_{2} \nonumber\\
	= & \| \sigma_{j}(A_{j} F_{2j-1}(X^{T}))^{T} \|_{2} \nonumber\\
	= & \| \sigma_{j}(A_{j} F_{2j-1}(X^{T}))^{T} - \sigma_{j}(0) \|_{2} \nonumber\\
	\le & \rho_{j} \| A_{j} F_{2j-1}(X^{T})^{T} - 0 \|_{2} \nonumber\\
	= & \rho_{j} \| A_{j} F_{2j-1}(X^{T})^{T} \|_{2} \nonumber\\
	\le & \rho_{j} \| A_{j} \|_{\sigma} \cdot \| F_{2j-1}(X^{T})^{T} \|_{2} ~.
\end{align}

%\item
(2) If there is a vine $V(2j-3, 2j+1, 1)$ connected at the vertex $N(2j+1)$, then we get the following inequality:
\begin{align}
\label{inductionVineResNet}
	& \| F_{2j+1}(X^{T})^{T} \|_{2} \nonumber\\
	= & \|  \sigma_{j}(A_{j} \sigma_{j}(A_{j} F_{2j-3}(X^{T})))^{T} + A^{2j-3, 2j+1, 1}_{1} F_{2j-3}(X^{T})^{T} \|_{2} \nonumber\\
	\le & \|  \sigma_{j}(A_{j} \sigma_{j}(A_{j} F_{2j-3}(X^{T})))^{T} \|_{2} + \| A^{2j-3, 2j+1, 1}_{1} F_{2j-3}(X^{T})^{T} \|_{2} \nonumber\\
	\le & \rho_{j} \| A_{j} \|_{\sigma} \rho_{j-1} \| A_{j-1} \|_{\sigma} \cdot \| F_{2j-3}(X^{T})^{T} \|_{2} + \| A^{2j-3, 2j+1, 1}_{1} \|_{\sigma} \cdot \| F_{2j-3}(X^{T})^{T} \|_{2} \nonumber\\
	= & \left(\rho_{j} \rho_{j-1} \| A_{j} \|_{\sigma} \cdot \| A_{j-1} \|_{\sigma} + \| A^{2j-3, 2j+1, 1}_{1} \|_{\sigma}\right) \| F_{2j-3}(X^{T})^{T} \|_{2} ~.
\end{align}
%\end{enumerate}
Therefore, based on eqs. (\ref{inductionNoVineResNet}) and (\ref{inductionVineResNet}), we can get the norm of output of ResNet as in the main text.

Similar with $\mathcal N_{2j+1}$, we can obtain an upper bound for the $\varepsilon_{u,v,1}$-covering number $\mathcal N_{V}^{u,v,1}$. Suppose the output computed at the vertex $N(u)$ is $F_{u}(X^{T})$. Then, we can get the following inequality:
\begin{equation}
\label{coverBoundVineResNet}
	\log \mathcal N_{V}^{u,v,1} \le \frac{(b^{u,v,1}_{1})^{2} \| F_{u}(X^{T})^{T} \|_{2}^{2}}{\varepsilon_{u,v,1}^{2}} \log(2W^{2}) ~.
\end{equation}

Applying eqs. (\ref{coverBoundStemResNet}) and (\ref{coverBoundVineResNet}) to eq. (\ref{coverBoundResNet}), we thus prove eq. (\ref{formulaCovBoundResNet}).

As for the formulation of the radiuses of the covers, we also employ an induction method.

%\begin{enumerate}
%\item
(1) Suppose the radius of the cover for the hypothesis space computed by the weight matrix $A_{1}$ and the nonlinearity $\sigma_{1}$ is $\varepsilon_{3}$. Then, applying eqs. (\ref{radiusInductionMatrix}) and (\ref{radiusInductionNonlinearity}), after the weight matrix $A_{2}$ and the nonlinearity $\sigma_{2}$, we get the following equation:
\begin{equation}
\label{formulaInductionHead}
	\varepsilon_{3} = (s_{2}+1)\rho_{2}\varepsilon_{1} ~.
\end{equation}

%\item
(2) Suppose the radius of the cover for the hypothesis space computed by the weight matrix $A_{j-1}$ and the nonlinearity $\sigma_{j-1}$ is $\varepsilon_{2j-1}$. Assume there is no vine connected around. Then, similarly, after the weight matrix $A_{2}$ and the nonlinearity $\sigma_{j}$, we get the following equation:
\begin{equation}
\label{formulaInductionStem}
	\varepsilon_{2j+1} = \rho_{j}(s_{j}+1)\varepsilon_{2j-1} ~.
\end{equation}

%\item
(3) Suppose the radius of the cover at the vertex $N(i)$ is $\varepsilon_{i}$. Assume there is a vine $V(u,u+4,1)$ links the stem at the vertex $N(u)$ and $N(u+4)$. Then, similarly, after the weight matrix $A_{2}$ and the nonlinearity $\sigma_{j}$, we get the following equation:
\begin{align}
\label{formulaInductionVine}
	\varepsilon_{2j+1} = & \varepsilon_{u+2}\left(s_{\frac{u-1}{2}} + 1\right)\rho_{\frac{u-1}{2}} + \varepsilon_{u}\left(s_{u,u+4,1} + 1\right) \nonumber\\
	= & \varepsilon_{u}\left(s_{\frac{u-1}{2}} + 1\right)\rho_{\frac{u-1}{2}}\left(s_{\frac{u-3}{2}} + 1\right)\rho_{\frac{u-3}{2}} + \varepsilon_{u}\left(s_{u,u+4,1} + 1\right) \nonumber\\
	= & \varepsilon_{u}\left(s_{\frac{u-1}{2}} + 1\right)\left(s_{\frac{u-3}{2}} + 1\right)\rho_{\frac{u-1}{2}}\rho_{\frac{u-3}{2}} + \varepsilon_{u}\left(s_{u,u+4,1} + 1\right) ~.
\end{align}
%\end{enumerate}

From eqs. (\ref{formulaInductionHead}), (\ref{formulaInductionStem}), and (\ref{formulaInductionVine}), we can obtain the following equation
\begin{align}
	\varepsilon = & \varepsilon_{1}\rho_{1}(s_{1}+1)\rho_{34}(s_{34}+1)\rho_{35} \prod_{\substack{1 \le i \le 16 \\ i \notin \{ 4,8,14 \}}} \left[ %\clubsuit 
	\rho_{2i}(s_{2i} + 1)\rho_{2i+1}(s_{2i+1} + 1) + 1 \right] \nonumber\\
	& \prod_{i \in \{ 4,8,14 \}} \left[ %\clubsuit 
	\rho_{2i}(s_{2i} + 1)\rho_{2i+1}(s_{2i+1} + 1) + s^{4i-1,4i+3,1}_{1} + 1 \right] ~.
\end{align}
%	where
%	\begin{equation}
		%\clubsuit 
%		(***) = \rho_{2i}(s_{2i} + 1)\rho_{2i+1}(s_{2i+1} + 1) ~.
%	\end{equation}
	
Combining the definition of $\bar \alpha$:
\begin{align}
	\bar \alpha = & \rho_{1}(s_{1}+1)\rho_{34}(s_{34}+1)\rho_{35} \prod_{\substack{ \le i \le 16 \\ i \notin \{ 4,8,14 \}}} \left[ %\clubsuit
	\rho_{2i}(s_{2i} + 1)\rho_{2i+1}(s_{2i+1} + 1) + 1 \right]  \nonumber\\
	& \prod_{i \in \{ 4,8,14 \}} \left[ %\clubsuit 
	\rho_{2i}(s_{2i} + 1)\rho_{2i+1}(s_{2i+1} + 1) + s^{4i-1,4i+3,1}_{1} + 1 \right] ~,
\end{align}
we can obtain that
\begin{equation}
	\varepsilon_{1} = \frac{\varepsilon}{\bar \alpha} ~.
\end{equation}
Applying eqs. (\ref{formulaInductionHead}), (\ref{formulaInductionStem}), and (\ref{formulaInductionVine}), we can get all $\varepsilon_{2j+1}$ and $\varepsilon^{u,u+4,1}$.

The proof is completed.
\end{proof}

\subsection{Generalization Bound for ResNet}
\label{sec:generalizartionBoundResNet}

\begin{proof}[Proof of Theorem \ref{generalizartionBoundResNet}]
We prove this theorem in $2$ steps: (1) We first apply Lemma \ref{covNumBound} to Lemma \ref{covBoundResNet} in order to get an upper bound on the Rademacher complexity of the hypothesis space computed by ResNet; and (2) We then apply the result of (1) to Lemma \ref{RadBound} in order to get a generalization bound.

%\begin{enumerate}
%\item
(1) {\it Upper bound on the Rademacher complexity.}

Applying eq. (\ref{formulaCovNumBound}) of Lemma \ref{covNumBound} to eq. (\ref{reFormulaCovBoundResNet}) of Lemma \ref{covBoundResNet}, we can get the following inequality:
\begin{align}
	\mathfrak R(\mathcal H_{\lambda}|_{D}) \le & \inf_{\alpha > 0} \left( \frac{4\alpha}{\sqrt{n}} + \frac{12}{n} \int_{\alpha}^{\sqrt n} \sqrt{\log \mathcal N(\mathcal H_{\lambda}|_{D}, \varepsilon, \| \cdot |_{2})} \text{d}\varepsilon \right) \nonumber\\
	\le & \inf_{\alpha > 0} \left( \frac{4\alpha}{\sqrt{n}} + \frac{12}{n} \int_{\alpha}^{\sqrt n} \frac{\sqrt{R}}{\varepsilon} \text{d}\varepsilon \right) \nonumber\\
	\le & \inf_{\alpha > 0} \left( \frac{4\alpha}{\sqrt{n}} + \frac{12}{n} \sqrt{R} \log\frac{\sqrt{n}}{\alpha} \right) ~.
\end{align}
Apparently, the infinimum is reached uniquely at $\alpha = 3\sqrt{\frac{R}{n}}$. Here, we use a simpler and also widely used choice $\alpha = \frac{1}{n}$, and get the following inequality:
\begin{equation}
%\label{inequalityRademacher}
	\mathfrak R(\mathcal H_{\lambda}|_{D}) \le \frac{4}{n^{\frac{3}{2}}} + \frac{18}{n} \sqrt{R} \log n ~.
\end{equation}

%\item
(2) {\it Upper bound on the generalization error.}

Combining with eq. (\ref{formulaRadBound}) of Lemma \ref{RadBound}, we get the following inequality:
\begin{align}
%\label{formulaRadBoundResNet}
	\Pr\{ \arg\max_{i} F(x)_{i} \ne y \} \le \hat{\mathcal R}_{\lambda}(F) + \frac{8}{n^{\frac{3}{2}}} + \frac{36}{n} \sqrt{R} \log n + 3 \sqrt{\frac{\log(1/\delta)}{2n}} ~.
\end{align}
%\end{enumerate}
The proof is completed.
\end{proof}

\section{Conclusion and Future Work}
\label{discussionConclusion}

We provide an upper bound for the covering number of the hypothesis space induced by deep neural networks with residual connections. The covering bound for ResNet, as an exemplary case, is then proposed. Combining various classic results in statistical learning theory, we further obtain a generalization bound for ResNet. With the generalization bound, we theoretically guarantee the performance of ResNet on unseen data. Considering the generality of our results, the generalization bound for ResNet can be easily extended to many state-of-the-art algorithms, such as DenseNet and ResNeXt.

This paper is based on the complexity of the whole hypothesis space. Some recent experimental results give an insight that SGD only explores a part of the hypothesis space and never visits other places. Thus, involving localisation properties into the analysis could lead to a tighter upper bound of the generalization error. However, there still lacks concrete evidence to support the localisation property, and the exact mechanism still remains an open problem. We plan to explore this problem in the future work.

\section*{Acknowledgment}
This work was supported by Australian Research Council under Grants FL170100117, DP180103424, IH180100002, and DE190101473.

\newpage
\bibliography{ResNet}
\end{document}